\documentclass[twoside]{article}

\usepackage[accepted]{aistats2014}
\usepackage{times}
\usepackage{natbib}
\usepackage{epsfig}
\usepackage{graphicx,subfigure}
\usepackage{amsmath,amsfonts,amsthm,amssymb}
\usepackage{algorithm}
\usepackage{algorithmic}

\newtheorem{theorem}{Theorem}
\newtheorem{lemma}[theorem]{Lemma}

\newtheorem{remark}{Remark}

\def\A{{\bf A}}
\def\a{{\bf a}}
\def\B{{\bf B}}
\def\bb{{\bf b}}
\def\C{{\bf C}}

\def\D{{\bf D}}

\def\I{{\bf I}}

\def\PP{{\bf P}}
\def\Q{{\bf Q}}

\def\S{{\bf S}}

\def\T{{\bf T}}
\def\U{{\bf U}}

\def\V{{\bf V}}

\def\W{{\bf W}}

\def\X{{\bf X}}
\def\x{{\bf x}}
\def\Y{{\bf Y}}

\def\0{{\bf 0}}
\def\1{{\bf 1}}

\def\IM{{\mathcal I}}

\def\OM{{\mathcal O}}
\def\PM{{\mathcal P}}

\def\RB{{\mathbb R}}
\def\RBmn{{\RB^{m\times n}}}
\def\EB{{\mathbb E}}
\def\PB{{\mathbb P}}

\def\Si{\mbox{\boldmath$\Sigma$\unboldmath}}

\def\Ome{\mbox{\boldmath$\Omega$\unboldmath}}

\def\argmin{\mathop{\rm argmin}}

\def\spann{\mathrm{span}}

\def\rk{\mathrm{rank}}
\def\diag{\mathsf{diag}}

\def\nystrom{{Nystr\"{o}m} }
\def\TimeMulti{{T_{\mathrm{Multiply}}}}

\begin{document}

%

%

\twocolumn[

\aistatstitle{Efficient Algorithms and Error Analysis for the Modified \nystrom Method}


\aistatsauthor{ Shusen Wang \And Zhihua Zhang }
\aistatsaddress{ College of Computer Science \& Technology\\ Zhejiang University, Hangzhou, China\\ wss@zju.edu.cn
             \And Department of Computer Science \& Engineering\\ Shanghai Jiao Tong University, Shanghai, China\\ zhihua@sjtu.edu.cn} ]

\begin{abstract}
    Many kernel methods suffer from high time and space complexities and are thus prohibitive in big-data applications. To tackle the computational challenge, the Nystr\"om method has been extensively used to reduce time and space complexities by sacrificing some accuracy. The Nystr\"om method speedups computation by constructing an approximation of the kernel matrix using only a few  columns of the matrix. Recently, a variant of the Nystr\"om method called the modified Nystr\"om method has demonstrated significant improvement over the standard Nystr\"om method in approximation accuracy, both theoretically and empirically.
    In this paper, we propose two algorithms that make the modified Nystr\"om method practical. First, we devise a  simple column selection algorithm with a provable error bound. Our algorithm is more  efficient and easier to implement than and nearly as accurate as the state-of-the-art algorithm. Second, with the selected columns at hand, we propose an algorithm that computes the approximation in lower time complexity than the approach in the previous work. Furthermore, we prove that the modified Nystr\"om method is exact under certain conditions, and we establish a lower error bound for the modified Nystr\"om method.
\end{abstract}

\section{Introduction}

The kernel method is an important tool in machine learning, computer vision, and data mining \citep{scholkopf2002learning,ShaweTaylorBook:2004}.
However, many kernel methods require matrix computations of high time and space complexities.
For example, let $m$ be the number of data instances.
The Gaussian process regression computes the inverse of an $m\times m$ matrix which takes time $\OM(m^3)$ and space $\OM(m^2)$;
the kernel PCA, Isomap, and Laplacian eigenmaps all perform the truncated singular value decomposition
which takes time $\OM(m^2 k)$ and space $\OM(m^2)$, where $k$ is the target rank of the decomposition.
When $m$ is large, it is challenging to store the $m\times m$ kernel matrix in RAM to perform these matrix computations.
Therefore, these kernel methods are prohibitive when $m$ is large.

To overcome the computational challenge, \cite{williams2001using} employed the \nystrom method \citep{nystrom1930praktische}
to generate a low-rank approximation to the original symmetric positive semidefinite (SPSD) kernel matrix.
By using the \nystrom method, eigenvalue decomposition and some matrix inverse can be approximately done on only a few columns of the SPSD matrix
instead of on the entire matrix, and the time and space costs are reduced to $\OM(m)$.
The \nystrom method has been widely used  to speedup various kernel methods, such as the
Gaussian process regression \citep{williams2001using},
spectral clustering \citep{fowlkes2004spectral,li2011time},
kernel SVMs \citep{zhang2008improved,yang2012nystrom},
kernel PCA \citep{zhang2008improved,zhang2010clustered,talwalkar2013large},
kernel ridge regression \citep{cortes2010impact,yang2012nystrom},
determinantal processes \citep{affandi2013nystrom}, etc.

To construct a low-rank matrix approximation, the \nystrom method requires a small number of columns (say, $c$ columns) to
be selected from the kernel matrix by a column sampling technique.
The approximation accuracy is largely determined by the sampling technique; that is,
a better sampling technique can result in a \nystrom approximate with a lower approximation error.
In the previous work much attention has been made on improving the error bounds of the \nystrom method:
additive-error bound has been explored by
\cite{drineas2005nystrom,Shawe-taylor05onthe,kumar2012sampling,jin2012improved}, etc.
Very recently, \cite{gittens2013revisiting} established the first relative-error bound
which is more interesting than additive-error bound \citep{mahoney2011ramdomized}.

However, the approximation quality cannot be arbitrarily improved by devising a very good sampling technique.
As shown theoretically by \cite{wang2013improving},
no matter what sampling technique is used to construct the \nystrom approximation,
the incurred error (in the spectral norm or the squared Frobenius norm) must grow with matrix size $m$ at least linearly.
Thus, the \nystrom approximation can be very rough when $m$ is large,
unless large number columns are selected.
As was pointed out by \cite{cortes2010impact}, the tighter kernel approximation leads to the better learning accuracy,
so it is useful to find a kernel approximation model that is more accurate than the \nystrom method.

To improve the approximation accuracy, \cite{wang2013improving} proposed
a new alternative called {\it the modified \nystrom method}
and a sampling algorithm for the modified \nystrom method.
The modified \nystrom method can be applied in  the same way exactly as the standard \nystrom method to speedup kernel methods.
The modified \nystrom method has an advantage that the error does not grow with matrix size $m$.
Therefore, by using the modified \nystrom method instead of the standard \nystrom method,
a significantly smaller number of columns is needed to attain the same accuracy as the standard \nystrom method.

However, it is much more expensive to construct the modified \nystrom approximation than to construct the standard standard \nystrom approximation.
Furthermore,
an efficient implementation of the modified \nystrom method keeps till open.
In this paper we seek to make the modified \nystrom method efficient and practical.

Additionally, 
\cite{kumar2009sampling,talwalkar2010matrix} showed that the standard \nystrom approximation is exact when the original kernel matrix is low-rank.
\cite{wang2013improving} proved the lower error bounds of the standard \nystrom method.
It is still open whether the modified \nystrom method has similar properties.
So we explore  the theoretical properties of the modified \nystrom method in this paper.

In sum, this paper offers the following contributions:
\vspace{-1mm}
\begin{itemize}
\item
    We devise a column selection algorithm with provable error bound for the modified \nystrom method.
    We call it {\it the uniform+adaptive$^2$ algorithm}.
    It is more efficient and much easier to implement than the near-optimal+adaptive algorithm of \cite{wang2013improving},
    yet its error bound is comparable with the near-optimal+adaptive algorithm.
\vspace{-2mm}
\item
    We provide an efficient algorithm for computing the intersection matrix of the modified \nystrom method.
    This algorithm can significantly reduce the time cost, especially when the kernel matrix is sparse.
\vspace{-2mm}
\item
    We show that the modified \nystrom approximation exactly recovers the original matrix under some conditions.
\vspace{-2mm}
\item
    We established a lower error bound for the modified \nystrom method.
    We conjecture that the lower error bound is tight.
\vspace{-2mm}
\end{itemize}

The remainder of this paper is organized as follows.
In Section~\ref{sec:notation} we define the notation used in this paper.
In Section~\ref{sec:previous} we formally define the \nystrom approximation methods and introduce some column sampling algorithms.
In Section~\ref{sec:efficient} we present an efficient column sampling algorithm and its error analysis.
In Section~\ref{sec:intersection} we devise an algorithm that computes the modified \nystrom approximation more efficiently.
In Section~\ref{sec:experiments} we empirically evaluate our proposed two algorithms.
In Section~\ref{sec:error_analysis} we explore some theoretical properties of the modified \nystrom method.

\section{Notation} \label{sec:notation}
\vspace{-2mm}

The notation used in this paper follows that of \cite{wang2013improving}.
For an $m {\times} n$ matrix $\A=[a_{i j}]$, we let $\a^{(i)}$ be its $i$-th row,
$\a_j$ be its $j$-th column,
$\|\A\|_F = (\sum_{i,j} a_{i j}^2)^{1/2}$ be its Frobenius norm, and
$\|\A\|_2 = \max_{\x\neq \0} \|\A \x\|_2 / \|\x\|_2$ be its spectral norm.

Letting $\rho=\rk(\A)$, we write the condensed singular value decomposition (SVD) of $\A$ as $\A = \U_\A \Si_\A \V_\A^T$,
where the $(i,i)$-th entry of $\Si_\A \in \RB^{\rho \times \rho}$ is the $i$-th largest singular value of $\A$.
We also let $\U_{\A,k}$ and $\V_{\A,k}$ be the first $k$ ($<\rho$) columns of $\U_\A$ and $\V_\A$, respectively,
and $\Si_{\A,k}$ be the $k\times k$ top sub-block of $\Si_\A$.
Then the $m\times n$ matrix $\A_k=\U_{\A,k} \Si_{\A,k} \V_{\A,k}^T$ is the ``closest'' rank-$k$ approximation to $\A$.

Based on SVD, the {\it matrix coherence} of the columns of $\A$ relative to the best rank-$k$ approximation to $\A$
is defined by $\mu_k = \frac{n}{k} \max_{j} \big\| \V_{\A,k}^{(j)} \big\|_2^2$.
Let $\A^\dag = \V_{\A} \Si_\A^{-1} \U_\A^T$ be the {\it Moore-Penrose inverse} of $\A$.
When $\A$ is nonsingular, the Moore-Penrose inverse is identical to the matrix inverse.
Given another ${m\times c}$ matrix $\C$,
we define $\PM_\C \A = \C \C^\dag \A$
as the projection of $\A$ onto the column space of $\C$
and $\PM_{\C,k} \A = \C \cdot \argmin_{\rk(\X)\leq k} \|\A - \C \X\|_F$
as the rank restricted projection.
It is obvious that $\|\A - \PM_\C \A\|_F \leq \|\A - \PM_{\C,k} \A\|_F$.

Finally, we discuss the time complexities of the matrix operations mentioned above.
For an $m{\times} n$ general matrix $\A$ (assume $m \geq n$),
it takes $\OM(m n^2)$ flops to compute the full SVD
and $\OM(m n k)$ flops to compute the truncated SVD of rank $k$ ($< n$).
The computation of  $\A^\dag$ takes $\OM(m n^2)$ flops.
It is worth mentioning that although multiplying an $m{\times} n$ matrix by an $n{\times} p$ matrix takes $m n p$ flops,
it can be performed in full parallel by partitioning the matrices into blocks.
Thus, the time and space expense of large-scale matrix multiplication is not a challenge in real-world applications.
We denote the time complexity of such a matrix multiplication by $T_{\mathrm{Multiply}}(m n p)$,
which can be tremendously smaller than $\OM(m n p)$ in parallel computing environment \citep{halko2011ramdom}.
An algorithm can still be efficient even if it demands large-scale matrix multiplications.

\section{Previous Work} \label{sec:previous}
\vspace{-2mm}

In Section~\ref{sec:previous:nystrom} we introduce the standard and  modified \nystrom methods
and discuss their advantages and disadvantages.
In Section~\ref{sec:previous:sampling} we describe some commonly used column sampling algorithms.

\subsection{The \nystrom Methods} \label{sec:previous:nystrom}

Given an $m\times m$ symmetric matrix $\A$,
one needs to select $c$ ($\ll m$) columns of $\A$ to form a matrix $\C \in \RB^{m\times c}$ to construct the standard or modified \nystrom approximation.
Without loss of generality,  $\A$ and $\C$ can be permuted such that
\begin{equation} \label{eq:definition_A}
\A \;=\;  \begin{bmatrix}
             \W & \A_{2 1}^T \\
             \A_{2 1} & \A_{2 2}
           \end{bmatrix}
          \quad \textrm{ and }
\quad
\C \;=\; \begin{bmatrix}
             \W  \\
             \A_{2 1}
           \end{bmatrix}\textrm{,}
\end{equation}
where $\W$ is of size $c\times c$.
The standard \nystrom approximation is defined by
\[
\tilde{\A}^{\textrm{nys}}_{c} \;\triangleq\; \C \U^{\textrm{nys}} \C^T \;=\; \C \W^\dag \C^T ,
\]
and the modified \nystrom approximation is
\[
\tilde{\A}^{\textrm{mod}}_{c} \;\triangleq \; \C \U^{\textrm{mod}} \C^T \;=\; \C \big( \C^\dag \A (\C^\dag)^T \big) \C^T .
\]
Here the $c\times c$ matrices $\U^{\textrm{nys}} \triangleq \W^\dag$ and $\U^{\textrm{mod}} \triangleq \C^\dag \A (\C^\dag)^T$
are called {\it the intersection matrices}. We see that
the only difference between the two models is their intersection matrices.

For the approximation $\C \U \C^T$ constructed by either of the methods,
given a target rank $k$, we hope the error ratio
\[
f=\|\A - \C \U \C^T\|_\xi /\|\A - \A_k\|_\xi , \quad (\xi=F \textrm{ or } 2),
\]
is as small as possible.
However, \cite{wang2013improving} showed that for the standard \nystrom method,
whatever a column selection algorithm is used,
the ratio $f$ must grow with the matrix size $m$ when $c$ is fixed.

\begin{lemma}[Lower Error Bound of the Standard \nystrom Method \citep{wang2013improving}]
Whatever a column sampling algorithm is used,
there exists an $m\times m$ SPSD matrix $\A$ such that the error incurred by the standard \nystrom method obeys:
\begin{eqnarray}
\big\| \A - \C \W^\dag \C^T \big\|_F^2 & \geq & \Ome \Big( 1 + \frac{m k}{c^2} \Big) \|\A - \A_k\|_F^2 ,\nonumber \\
\big\| \A - \C \W^\dag \C^T \big\|_2 & \geq & \Ome \Big( \frac{m}{c} \Big) \|\A - \A_k\|_2 . \nonumber
\end{eqnarray}
Here $k$ is an arbitrary target rank, and $c$ is the number of selected columns.
\end{lemma}

Thus, when the matrix size $m$ is large, the standard \nystrom approximation is very inaccurate unless a large number of columns are selected.
By comparison, when using an algorithm in \cite{wang2013improving} for the modified \nystrom method,
the error ratio $f$ remains constant for a fixed $c$ and a growing $m$.
Therefore, the modified \nystrom method is more accurate than the standard \nystrom method.

However, the accuracy gained by the modified \nystrom method is at the cost of higher time and space complexities.
Computing the intersection matrix $\U^{\textrm{nys}} = \W^\dag$ only takes time  $\OM(c^3)$  and space $\OM(c^2)$,
while computing $\U^{\textrm{mod}} = \C^\dag \A (\C^\dag)^T$ naively takes time $\OM(m c^2) + \TimeMulti(m^2 c)$
and space $\OM(m c)$~\footnote{The matrix multiplication can be done blockwisely, that is,
loading two small blocks into RAM to perform multiplication at a time.
So the space cost of the matrix multiplication is $\OM(m c)$ rather than $\OM(m^2)$ \citep{wang2013improving}.}.

\subsection{Sampling Algorithms for the \nystrom Methods} \label{sec:previous:sampling}

The column selection problem has been widely studied in the theoretical computer science community
\citep{boutsidis2011NOCshort,mahoney2011ramdomized,Guruswami2012optimal}
and the numerical linear algebra community \citep{gu1996efficient,stewart1999four},
and numerous algorithms have been devised and analyzed.
Here we focus on some theoretically guaranteed algorithms studied in the theoretical computer science community.

In the previous work much attention has been paid on improving column sampling algorithms
such that the \nystrom approximation is more accurate.
Uniform sampling is the simplest and most time-efficient column selection algorithm,
and it has provable error bounds when applied to the standard \nystrom method \citep{gittens2011spectral,jin2012improved,kumar2012sampling,gittens2013revisiting}.
To improve the approximation accuracy, many importance sampling algorithms have been proposed,
among which the adaptive sampling of \cite{deshpande2006matrix} (see Algorithm~\ref{alg:adaptive})
and the leverage score based sampling of \cite{drineas2008cur,ma2014statistical} are widely studied.
The leverage score based sampling has provable bounds when applied to the standard \nystrom method \citep{gittens2013revisiting},
and the adaptive sampling has provable bounds when applied to the modified \nystrom method \citep{wang2013improving}.
Besides, quadratic R\'enyi entropy based active subset selection~\citep{de2010optimized}
and $k$-means clustering based selection~\citep{zhang2010clustered} are also effective algorithms,
but they do not have additive-error or relative-error bound.

Particularly, \cite{wang2013improving} proposed an algorithm for the modified \nystrom method
by combining the near-optimal column sampling algorithm \citep{boutsidis2011NOCshort}
and the adaptive sampling algorithm \citep{deshpande2006matrix}.
The error bound of the algorithm is the strongest  among all the feasible algorithms for the \nystrom methods.
We show it in the following lemma.

\begin{lemma}[The Near-Optimal+Adaptive Algorithm \citep{wang2013improving}] \label{lem:near_opt}
Given a symmetric matrix $\A \in \RB^{m\times m}$ and a target rank $k$,
the algorithm samples totally
$c = \OM( k \epsilon^{-2} )$
columns of $\A$ to construct the approximation.
We run the algorithm $t \geq (2\epsilon^{-1} + 1) \log (1/p)$ times (independently in parallel)
and choose the sample that minimizes $\|\A - \C \big( \C^\dag \A (\C^{\dag})^T \big) \C^T \big\|_F$,
then the inequality
\begin{equation}
\big\|\A - \C \big( \C^\dag \A (\C^\dag)^T \big) \C^T \big\|_F
\;\leq\; (1+\epsilon) \|\A - \A_k \|_F  \nonumber
\end{equation}
holds with probability at least $1-p$.
The algorithm costs $\OM\big( m c^2 + m k^3 \epsilon^{-2/3} \big) + \TimeMulti\big( m^2 c \big)$ time and $\OM(m c)$ space in
computing $\C$ and $\U$.
\end{lemma}

The near-optimal+adaptive algorithm is effective and efficient, but its implementation is very complicated.
Its main component---the near-optimal column selection algorithm---consists of three steps:
approximate SVD via random projection~\citep{boutsidis2011NOCshort,halko2011ramdom},
the dual-set sparsification algorithm~\citep{boutsidis2011NOCshort},
and the adaptive sampling algorithm~\citep{deshpande2006matrix}.
Without careful implementation of the first two steps,
the time and space costs roar, making  the near-optimal+adaptive algorithm  inefficient.

\begin{algorithm}[tb]
   \caption{The Uniform+Adaptive${}^2$ Algorithm.}
   \label{alg:efficient}
\algsetup{indent=2em}
\begin{small}
\begin{algorithmic}[1]
   \STATE {\bf Input:} an $m\times m$ symmetric matrix $\A$, target rank $k$, error parameter $\epsilon \in (0, 1]$, matrix coherence $\mu$.
   \STATE {\bf Uniform Sampling.} Uniformly sample
            \vspace{-2mm}
            \[c_1 = 8.7 \mu k \log \big(\sqrt{5} k \big)
            \vspace{-2mm}
            \]
            columns of $\A$ without replacement to construct $\C_1$; \label{alg:practical:uniform}
   \STATE {\bf Adaptive Sampling.} Sample
            \vspace{-2mm}
            \[c_2 = 10 k \epsilon^{-1}
            \vspace{-2mm}
            \]
            columns of $\A$ to construct $\C_2$
            using adaptive sampling algorithm~\ref{alg:adaptive} according to the residual $\A - \PM_{\C_1} \A$; \label{alg:practical:adapt1}
   \STATE {\bf Adaptive Sampling.} Sample
            \vspace{-2mm}
            \[c_3 = 2 \epsilon^{-1} (c_1 + c_2)
            \vspace{-2mm}
            \] columns of $\A$ to construct $\C_3$ \label{alg:practical:adapt2}
            using adaptive sampling algorithm~\ref{alg:adaptive} according to the residual $\A - \PM_{[\C_1,\;\C_2]} \A$;
   \RETURN $\C = [\C_1 , \C_2 , \C_3]$ and $\U = \C^\dag \A (\C^\dag)^T$.
\end{algorithmic}
\end{small}
\end{algorithm}

\section{An Efficient Column Sampling Algorithm for the Modified \nystrom Method} \label{sec:efficient}

In this paper we propose a column sampling algorithm
which is efficient, effective, and very easy to implement.
The algorithm consists of a uniform sampling step and two adaptive sampling steps, so we call it {\it the uniform+adaptive${}^2$ algorithm}.
The algorithm is described in Algorithm~\ref{alg:efficient} and analyzed in Theorem~\ref{thm:efficient}.

The idea behind the uniform+adaptive${}^2$ algorithm is quite intuitive.
Since the modified \nystrom method is the simultaneous projection of $\A$ onto the column space of $\C$ and the row space of $\C^T$,
the approximation error will get lower if $\spann(\C)$ better approximates $\spann(\A)$.
After the initialization by uniform sampling,
the columns of $\A$ far from $\spann(\C_1)$ have large residuals and are thus likely to get chosen by the adaptive sampling.
After two rounds of adaptive sampling, columns of $\A$ are likely to be near $\spann(\C)$.

It is worth  mentioning that our uniform+adaptive${}^2$ algorithm is similar to
the adaptive-full algorithm of \cite[Figure 3]{kumar2012sampling}.
The adaptive-full algorithm consists of a random initialization followed by multiple adaptive sampling steps.
Obviously, using multiple adaptive sampling steps can surely reduce the approximation error.
However, the update of sampling probability in each step is expensive, so we choose to do only two steps.
Importantly, the adaptive-full algorithm of \cite[Figure 3]{kumar2012sampling} is merely a heuristic scheme without theoretical guarantee,
whereas our uniform+adaptive${}^2$ algorithm has a strong error bound which is nearly as good as the state-of-the-art algorithm of \cite{wang2013improving}
(See Theorem~\ref{thm:efficient}).

\begin{algorithm}[tb]
   \caption{The Adaptive Sampling Algorithm.}
   \label{alg:adaptive}
\algsetup{indent=2em}
\begin{small}
\begin{algorithmic}[1]
   \STATE {\bf Input:} a residual matrix $\B \in \RB^{m\times n}$ and number of selected columns $c$ $(< n)$.
   \STATE Compute sampling probabilities $p_j = \|\bb_j\|_2^2 / \|\B\|_F^2$ for $j=1, \cdots , n$;
   \STATE Select $c$ indices in $c$ i.i.d.\ trials, in each trial the index $j$ is chosen with probability $p_j$;
   \RETURN an index set containing the indices of the selected columns.
\end{algorithmic}
\end{small}
\end{algorithm}

\begin{theorem}[The Uniform+Adaptive${}^2$ Algorithm.] \label{thm:efficient}
Given an $m {\times} m$ symmetric matrix $\A$ and a target rank $k$, we let $\mu_k$ denote the matrix coherence of $\A$.
Algorithm~\ref{alg:efficient} samples totally
\[
c = \OM\big( k \epsilon^{-2} + \mu_k \epsilon^{-1} k \log k \big)
\]
columns of $\A$ to construct the approximation.
We run Algorithm~\ref{alg:efficient}
\[
t \geq (20\epsilon^{-1} + 18) \log (1/p)
\]
times (independently in parallel)
and choose the sample that minimizes $\|\A - \C \big( \C^\dag \A (\C^{\dag})^T \big) \C^T \big\|_F$,
then the inequality
\begin{eqnarray}
\big\|\A - \C \big( \C^\dag \A (\C^{\dag})^T \big) \C^T \big\|_F
\;\leq \; \big(1 + \epsilon \big) \big\| \A - \A_k \big\|_F \nonumber
\end{eqnarray}
holds with probability at least $1-p$.
The algorithm costs $\OM\big(m c^2 \big) + T_{\mathrm{Multiply}}\big(m^2 c\big)$ time and $\OM(m c)$ space in
computing $\C$ and $\U$.
\end{theorem}

\begin{remark}\label{remark:efficient}
Theoretically, Algorithm~\ref{alg:efficient} requires to compute the matrix coherence of $\A$
in order to determine $c_1$, $c_2$, and $c_3$.
However, computing the matrix coherence takes time $\OM(m^2 k)$ and is thus impractical;
even the fast approximation approach of \cite{drineas2012fast} is not feasible here because $\A$ is a square matrix.
The use of the matrix coherence here is merely for theoretical analysis;
setting the parameter $\mu$ in Algorithm~\ref{alg:efficient} to be exactly the matrix coherence does not certainly result in the highest accuracy.
According to our off-line experiments, the resulting approximation accuracy is not sensitive to the value of $\mu$.
So we strongly suggest the users to set $\mu$ in Algorithm~\ref{alg:efficient} to be a constant
rather than actually computing the matrix coherence.
\end{remark}

Table~\ref{tab:algorithms} presents comparisons between the near-optimal+adaptive algorithm of \cite{wang2013improving} and our uniform+adaptive$^2$ algorithm.
The time complexity of our algorithm is lower than the near-optimal+adaptive algorithm,
and the space complexities of the two algorithms are the same.
To attain the same error bound, our algorithm needs to select $c=\OM\big( k \epsilon^{-2}+ \mu_k \epsilon^{-1} k \log k \big)$ columns,
which is a little larger than that of the near-optimal+adaptive algorithm.
When $\epsilon \rightarrow 0$, we have that
$\OM\big( k \epsilon^{-2}+ \mu_k \epsilon^{-1} k \log k \big) = \OM\big( k \epsilon^{-2})$.
Therefore, the error bound of our algorithm is nearly as good as the near-optimal+adaptive algorithm
because $\epsilon$ is usually set to be a very small value.

\begin{table}[t]\setlength{\tabcolsep}{0.3pt}
\caption{Comparisons between the two sampling algorithms in time complexity, space complexity,
        the number of selected columns, and the hardness of implementation.}
\label{tab:algorithms}
\begin{center}
\begin{footnotesize}
\begin{tabular}{p{32pt} p{105pt} p{105pt}}
\hline
        &	\multicolumn{1}{c}{Uniform+Adaptive$^2$ ~~}	& \multicolumn{1}{c}{~~Near-Optimal+Adaptive}  \\
\hline
Time    & \multicolumn{1}{c}{$\OM\big( m c^2 \big)+ \TimeMulti\big( m^2 c \big)$}
        & \multicolumn{1}{c}{$\OM\big( m c^2 +m k^3 \epsilon^{-2/3} \big)$ }  \\ [-2pt]
        & & \multicolumn{1}{c}{$+ \TimeMulti\big( m^2 c \big)$} \\ [3pt]
Space   & \multicolumn{1}{c}{$\OM\big( m c \big)$}
        & \multicolumn{1}{c}{$\OM\big( m c \big)$} \\ [3pt]
\#columns & \multicolumn{1}{c}{$\OM\big( k \epsilon^{-2}+ \mu_k \epsilon^{-1} k \log k \big)$}
        & \multicolumn{1}{c}{$\OM\big( k \epsilon^{-2} )$}\\ [3pt]
Implement  & \multicolumn{1}{c}{Easy to implement}
        & \multicolumn{1}{c}{Hard to implement}  \\
\hline
\end{tabular}
\end{footnotesize}
\end{center}
\end{table}

\section{Fast Computation of the Intersection Matrix} \label{sec:intersection}

Naively computing the intersection matrix $\U = \C^\dag \A (\C^\dag)^T$ takes time $\OM(m c^2) + T_{\mathrm{Multiply}} ( m^2 c)$,
which is much more expensive than computing $\W^\dag$ for the standard \nystrom method.
In this section we propose a more efficient algorithm for computing the intersection matrix,
which  only takes time $\OM(c^3) + \TimeMulti\big((m-c)^2 c\big)$.
The algorithm is described in Theorem~\ref{thm:intersection_matrix}.
The algorithm is obtained by expanding the Moore-Penrose inverse of $\C$ using the theorem in~\citep[Page~179]{adi2003inverse}.

%
%

\begin{theorem} \label{thm:intersection_matrix}
For an $m\times m$ symmetric matrix $\A$,
when the submatrix $\W$ is nonsingular,
the intersection matrix of the modified \nystrom method $\U = \C^\dag \A (\C^\dag)^T$
can be computed in time $\OM(c^3) + T_{\mathrm{Multiply}} \big( (m-c)^2 c\big)$ by the following formula:
\vspace{-1mm}
\[
\U \;=\; \C^\dag \A (\C^\dag)^T
\;=\; \T_1 \big( \W + \T_2 + \T_2^T + \T_3 \big) \T_1^T \textrm{,}
\vspace{-1mm}
\]
where the intermediate matrices are computed by
\vspace{-1mm}
\begin{eqnarray}
\T_0 \;=\; \A_{2 1}^T \A_{2 1} \textrm{,}&   &
\T_1 \;=\; \W^{-1} \big( \I_c + \W^{-1} \T_2 \big)^{-1} \textrm{,}  \nonumber \\
\T_2 \;=\; \T_0 \W^{-1}        \textrm{,}&   &
\T_3 \;=\; \W^{-1} \big(\A_{2 1}^T \A_{2 2} \A_{2 1} \big)\W^{-1} \textrm{.}  \nonumber
\vspace{-1mm}
\end{eqnarray}
The four intermediate matrices are all of size $c\times c$,
and the matrix inverse operations are on $c\times c$ small matrices.
\end{theorem}

\begin{remark}
Since the submatrix $\W$ is not in general nonsingular,
before using the algorithm, the user should first test the rank of $\W$, which takes time $\OM(c^3)$.
Empirically, for graph Laplacian and the radial basis function (RBF)  kernel \citep{Genton2001classes},
the submatrix $\W$ is usually nonsingular, and the algorithm is useful;
for the linear kernel, $\W$ is often singular, so the algorithm does not work.
\end{remark}

\begin{figure*}[!ht]
\subfigtopskip = 0pt
\begin{center}
\centering
\includegraphics[width=56mm,height=33mm]{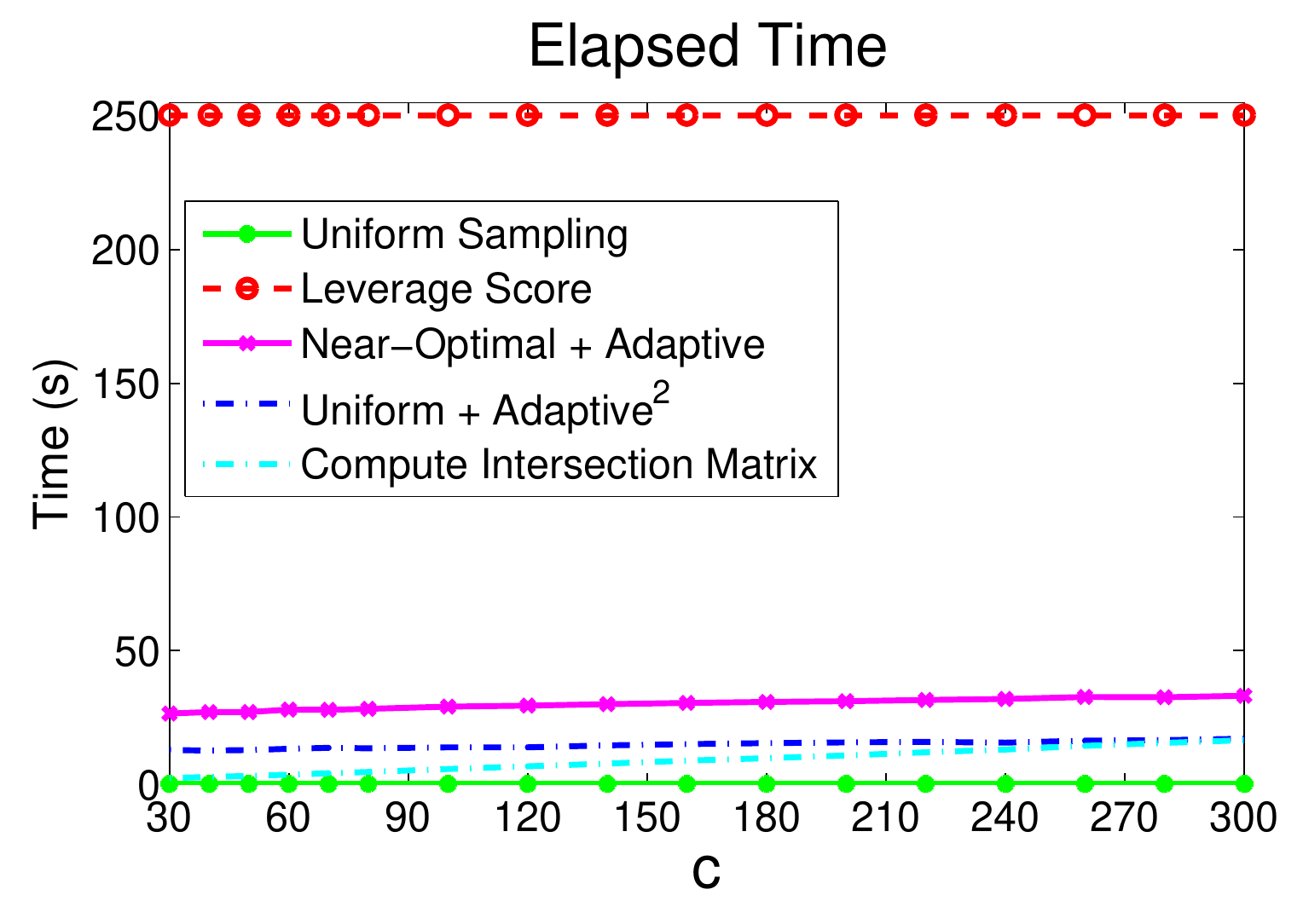}
\includegraphics[width=56mm,height=33mm]{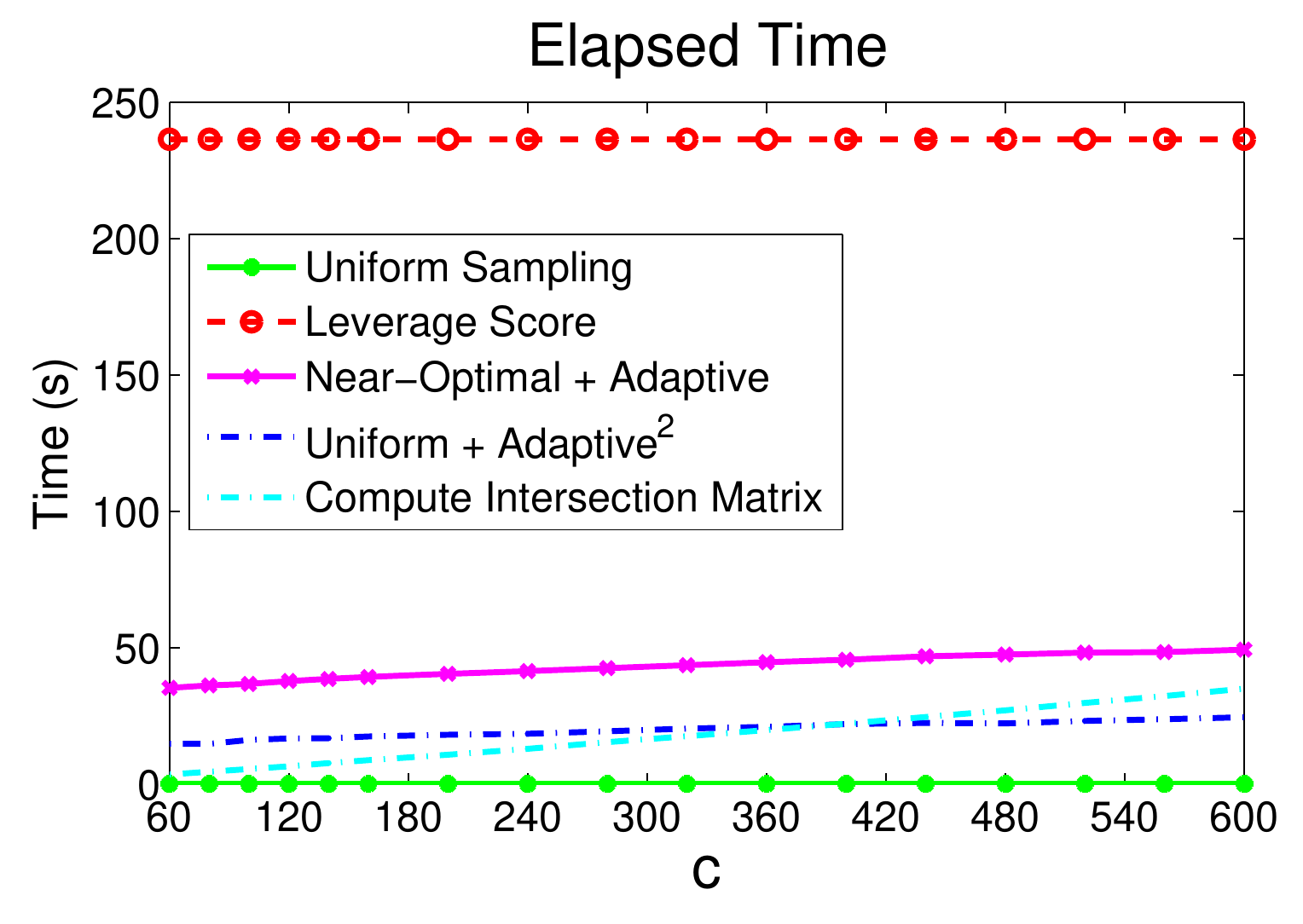}
\includegraphics[width=56mm,height=33mm]{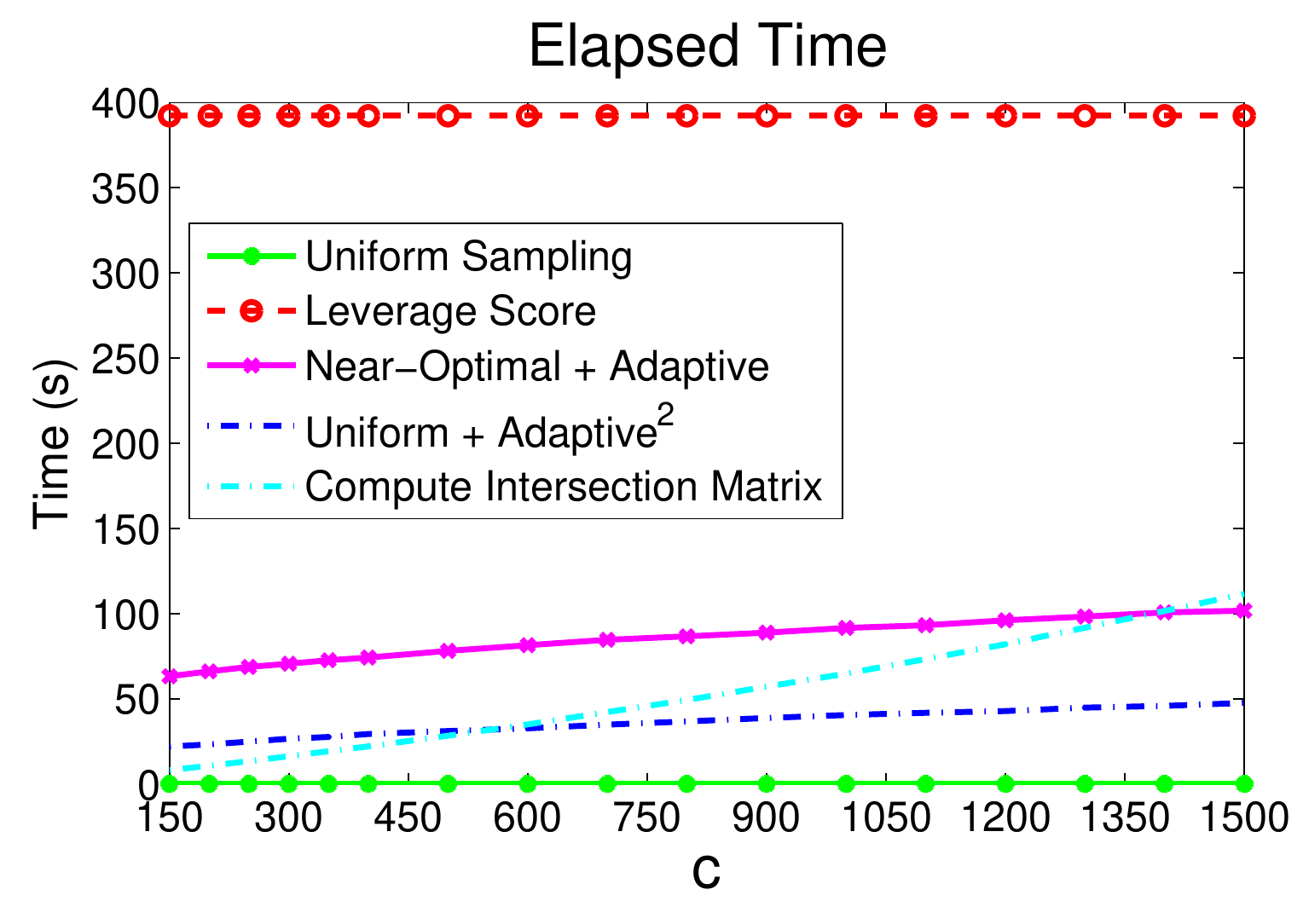}\\
\subfigure[$k = 10$]{\includegraphics[width=56mm,height=33mm]{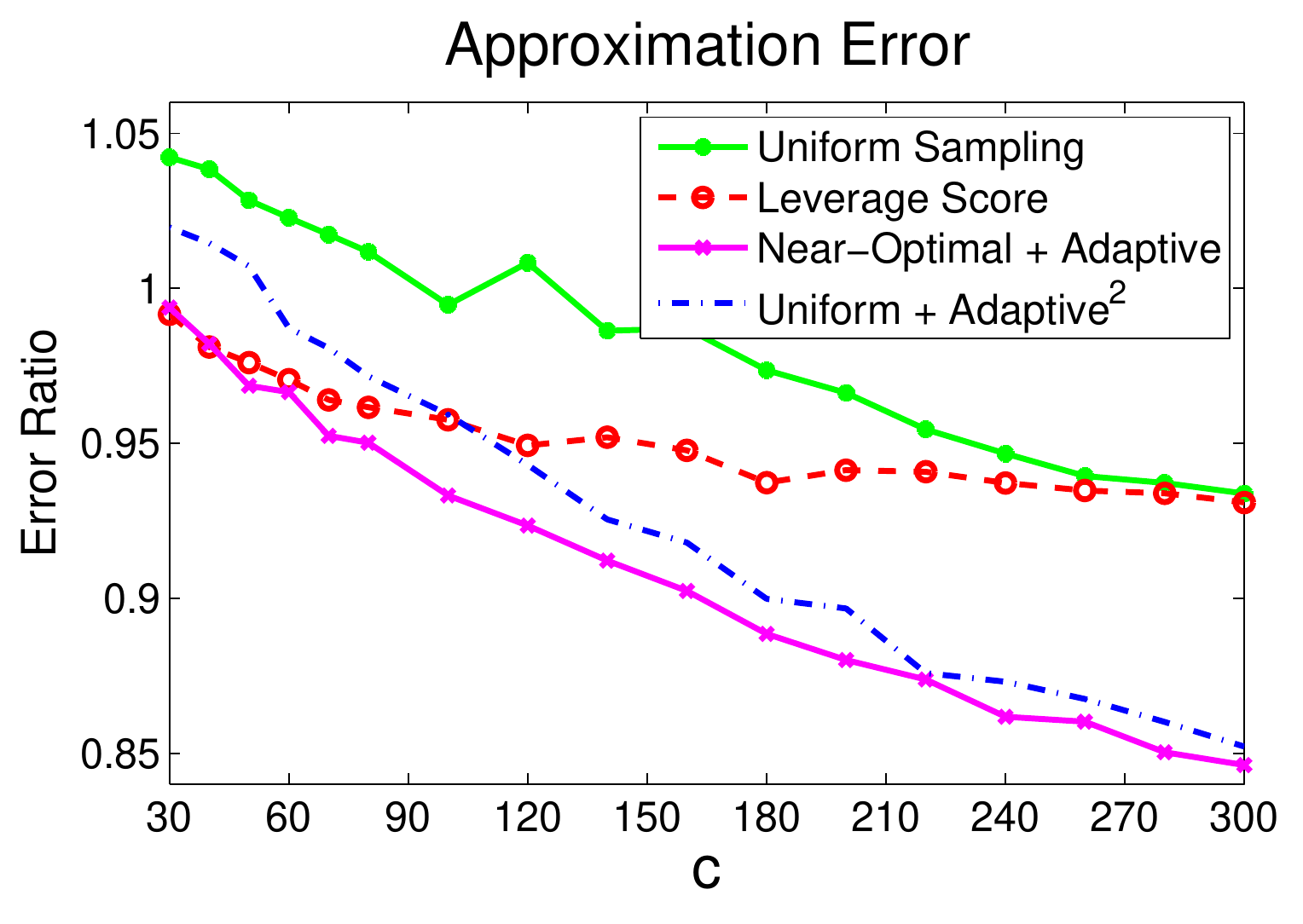}}
\subfigure[$k = 20$]{\includegraphics[width=56mm,height=33mm]{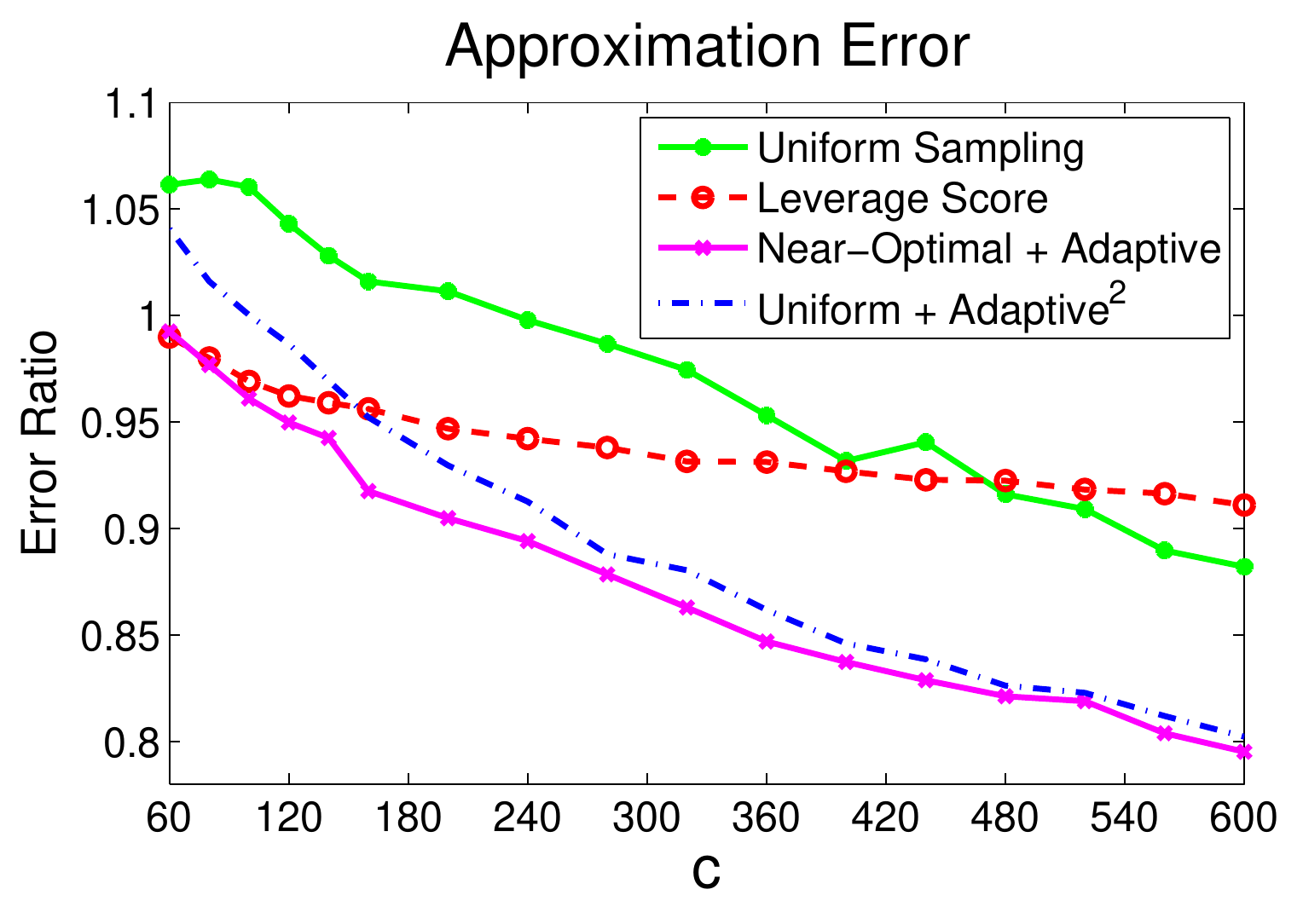}}
\subfigure[$k = 50$]{\includegraphics[width=56mm,height=33mm]{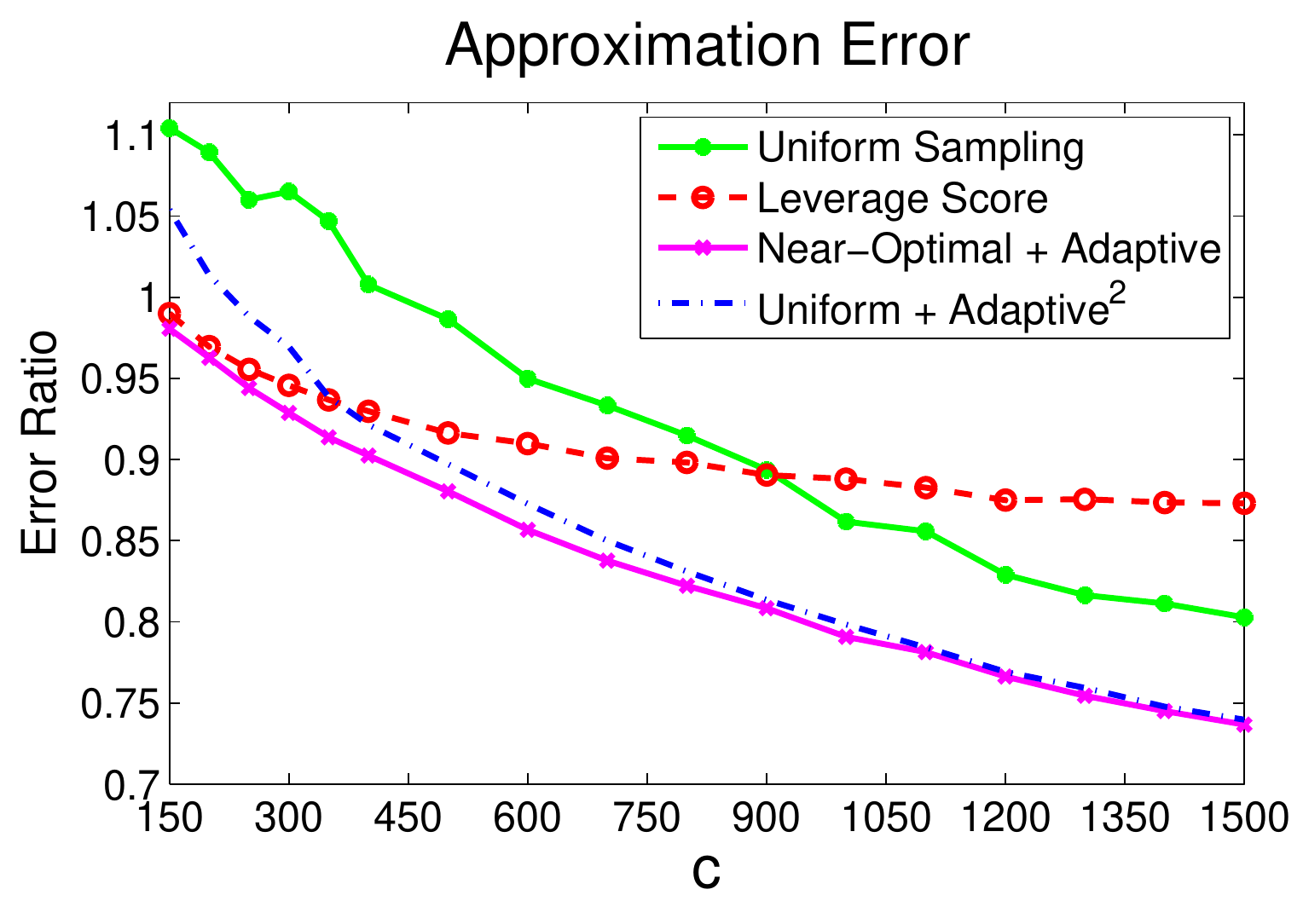}}
\end{center}
   \caption{Results on the RBF kernel of the Letters dataset.
            Here the matrix coherence of the kernel matrix is $\mu_{10} = 62.05$, $\mu_{20} = 34.87$, and $\mu_{50} = 19.16$.}
\label{fig:letter}
\end{figure*}

\begin{figure*}[!ht]
\subfigtopskip = 0pt
\begin{center}
\centering
\includegraphics[width=56mm,height=33mm]{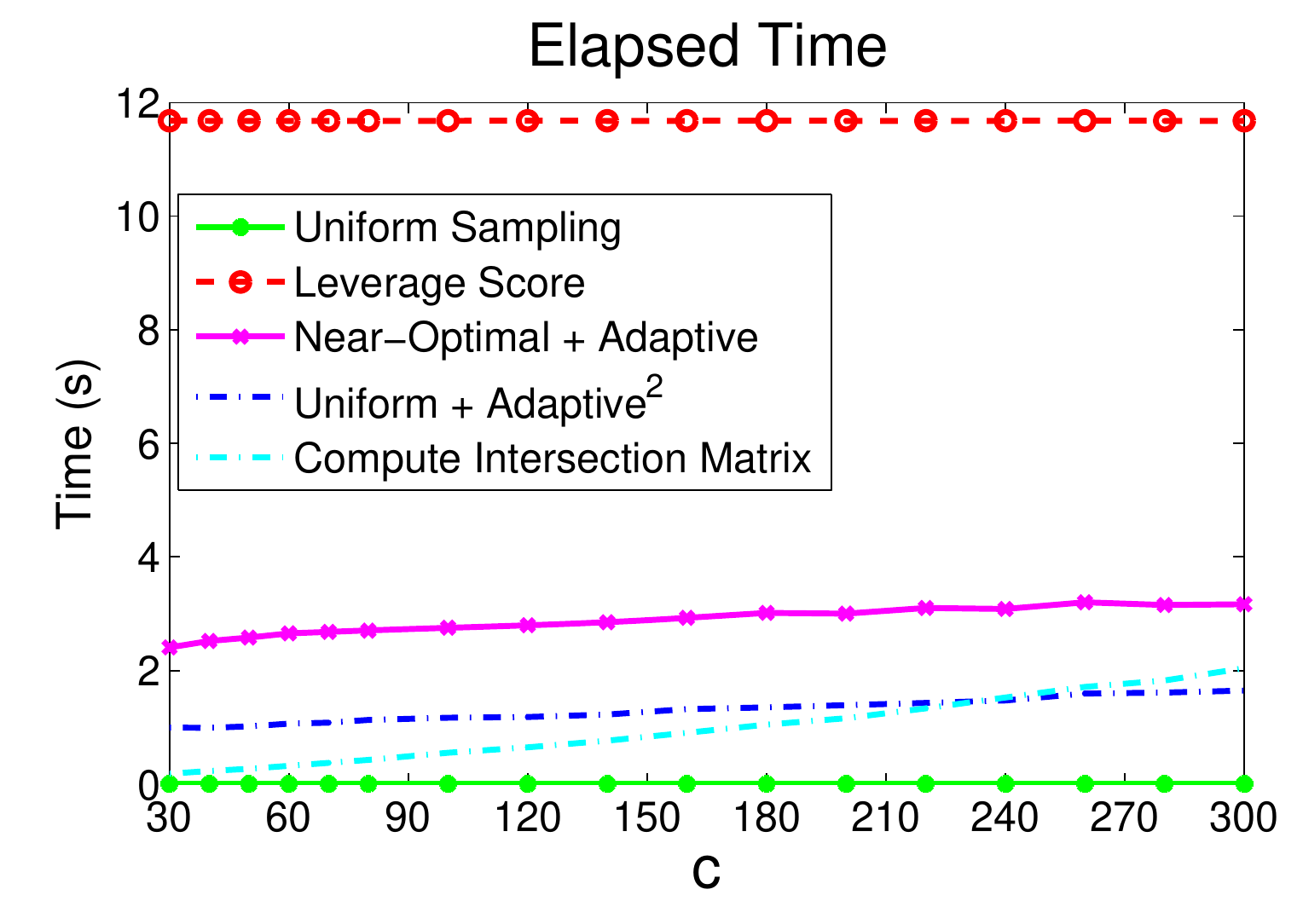}
\includegraphics[width=56mm,height=33mm]{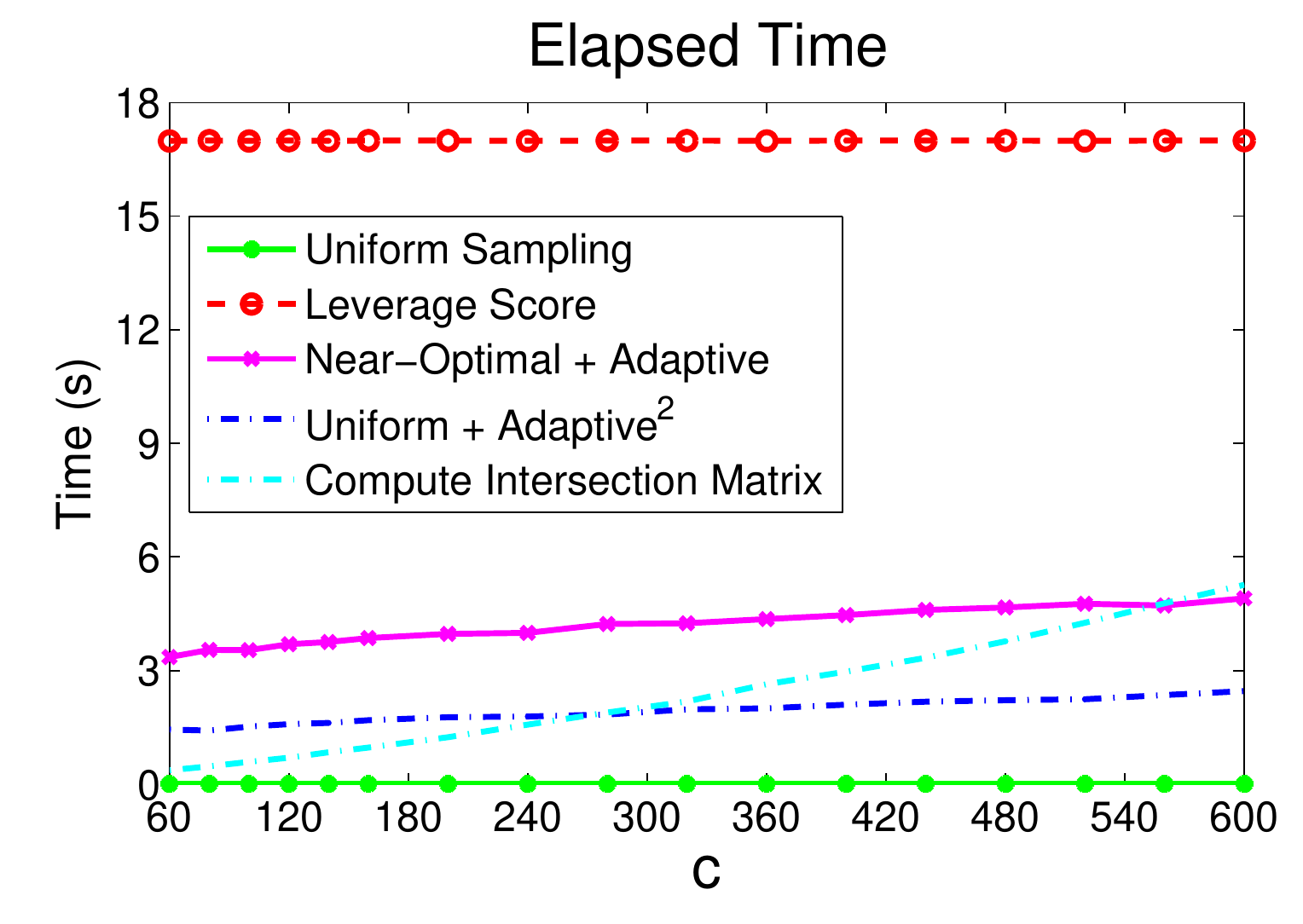}
\includegraphics[width=56mm,height=33mm]{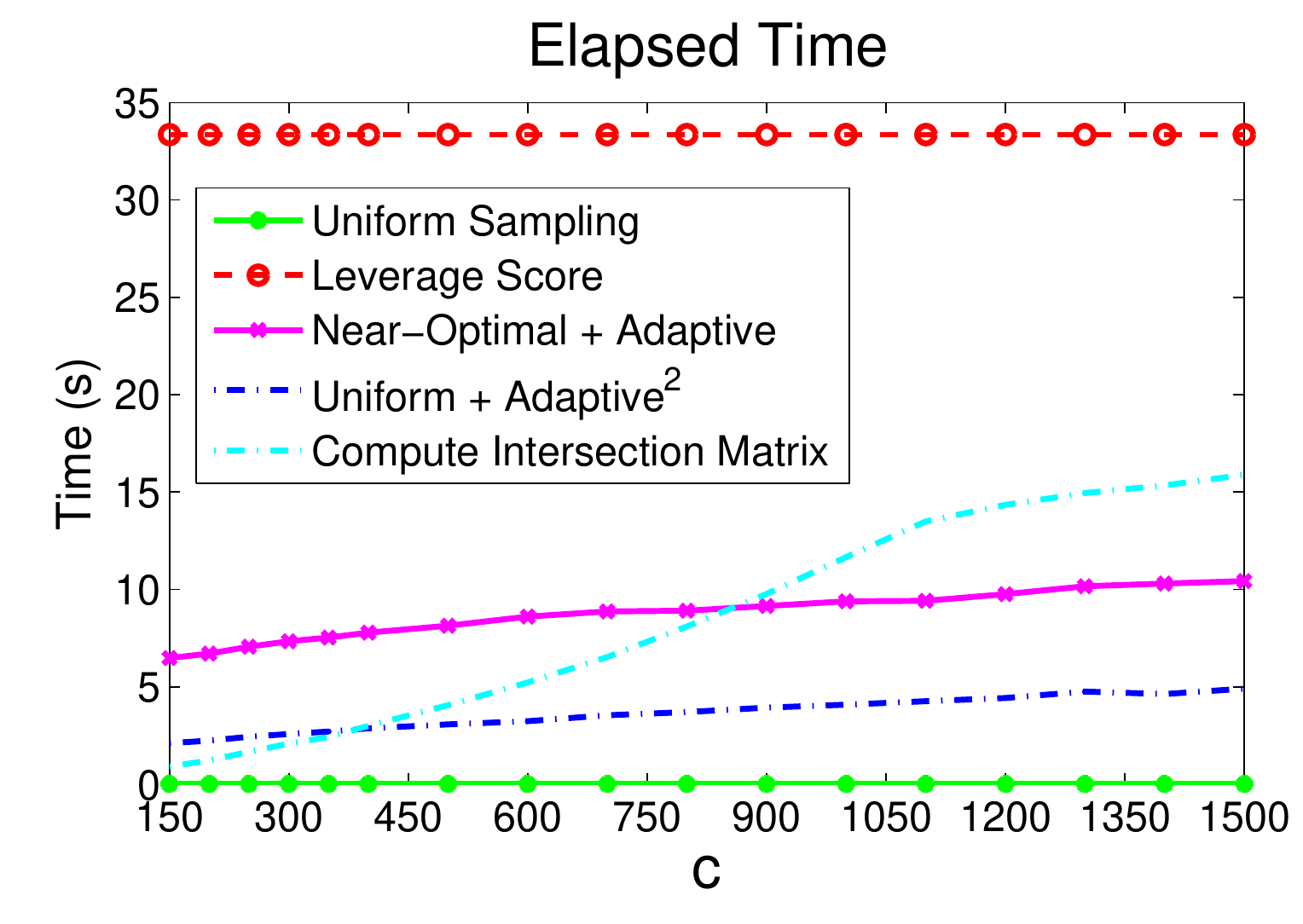}\\
\subfigure[$k = 10$]{\includegraphics[width=56mm,height=33mm]{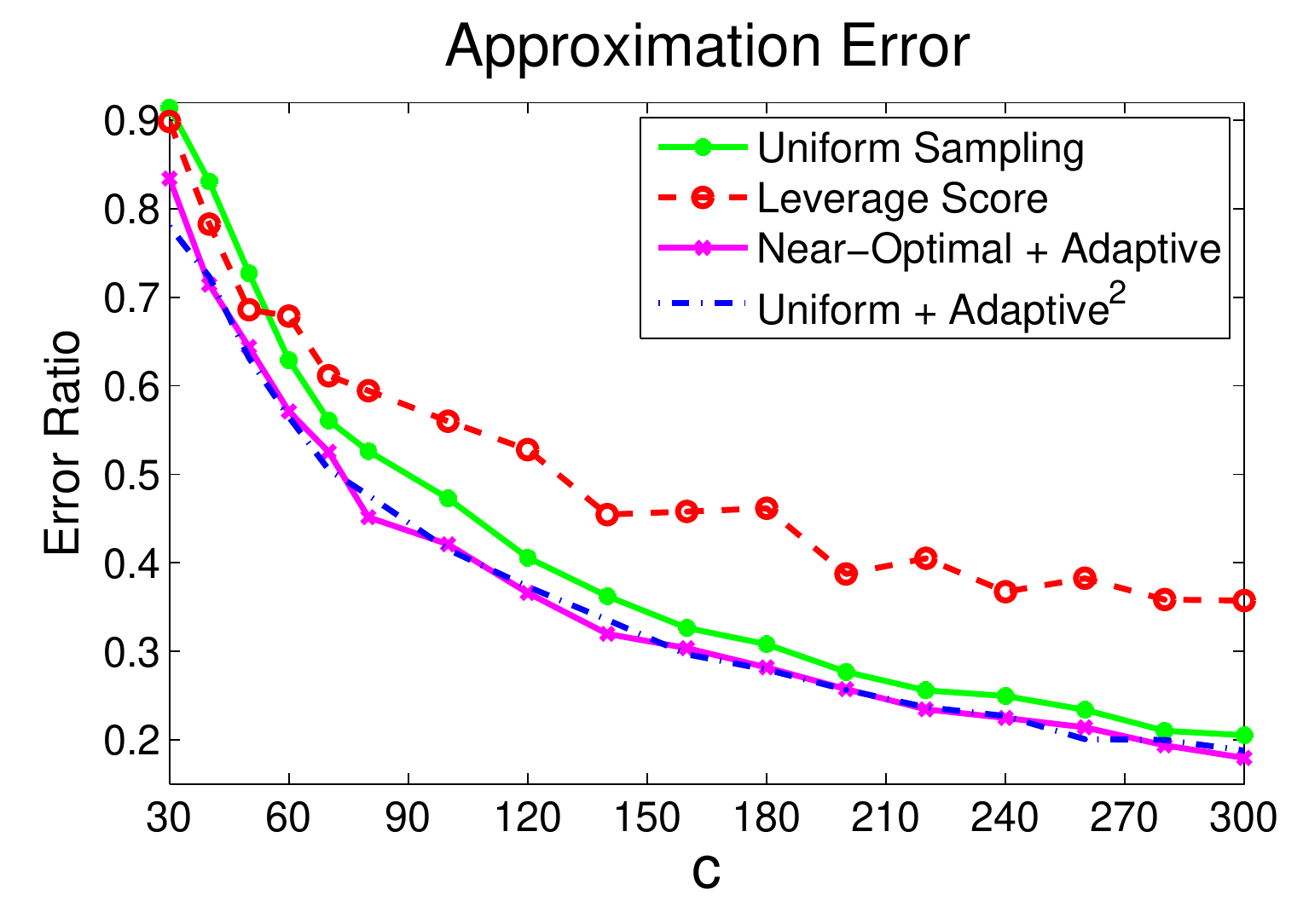}}
\subfigure[$k = 20$]{\includegraphics[width=56mm,height=33mm]{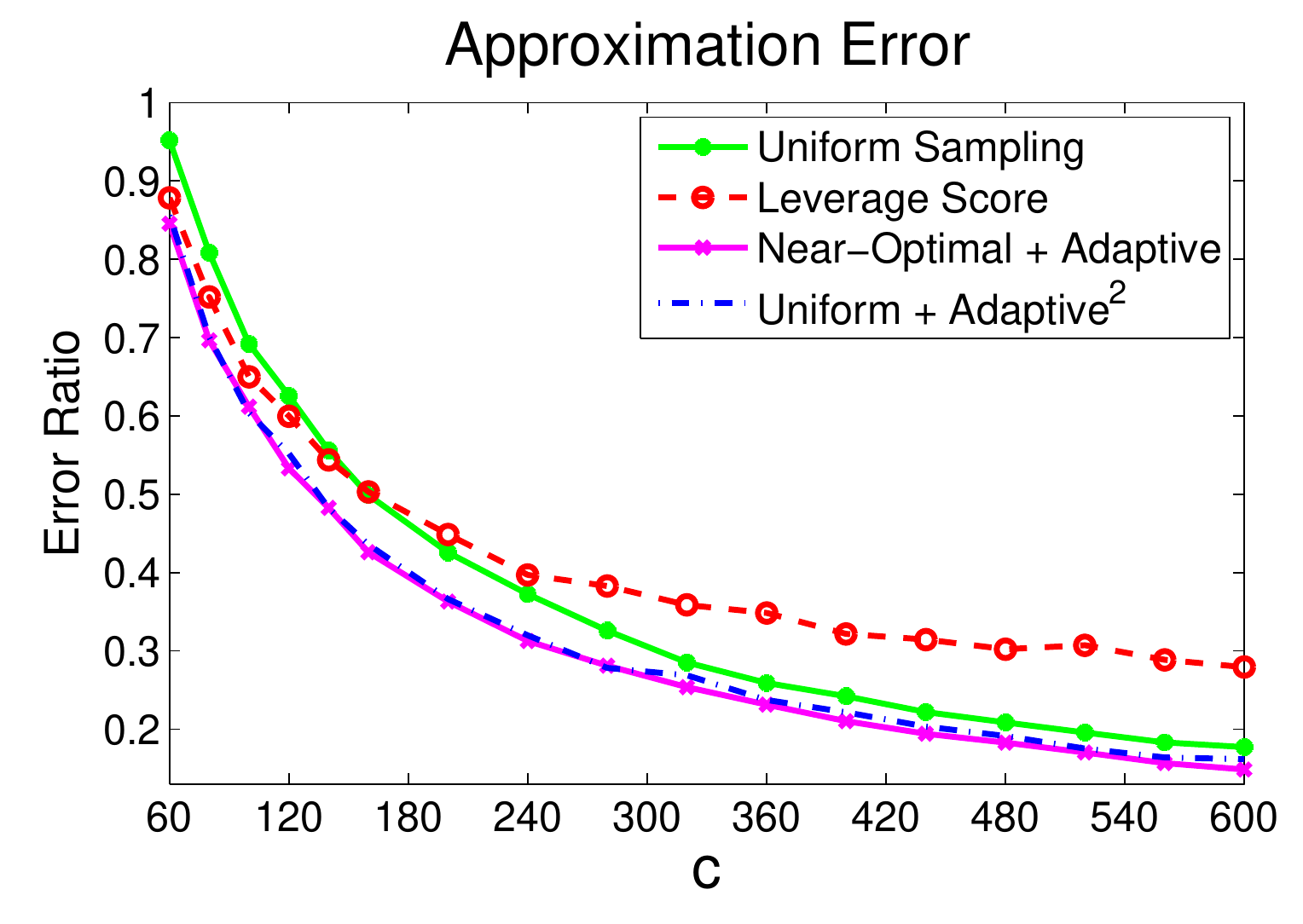}}
\subfigure[$k = 50$]{\includegraphics[width=56mm,height=33mm]{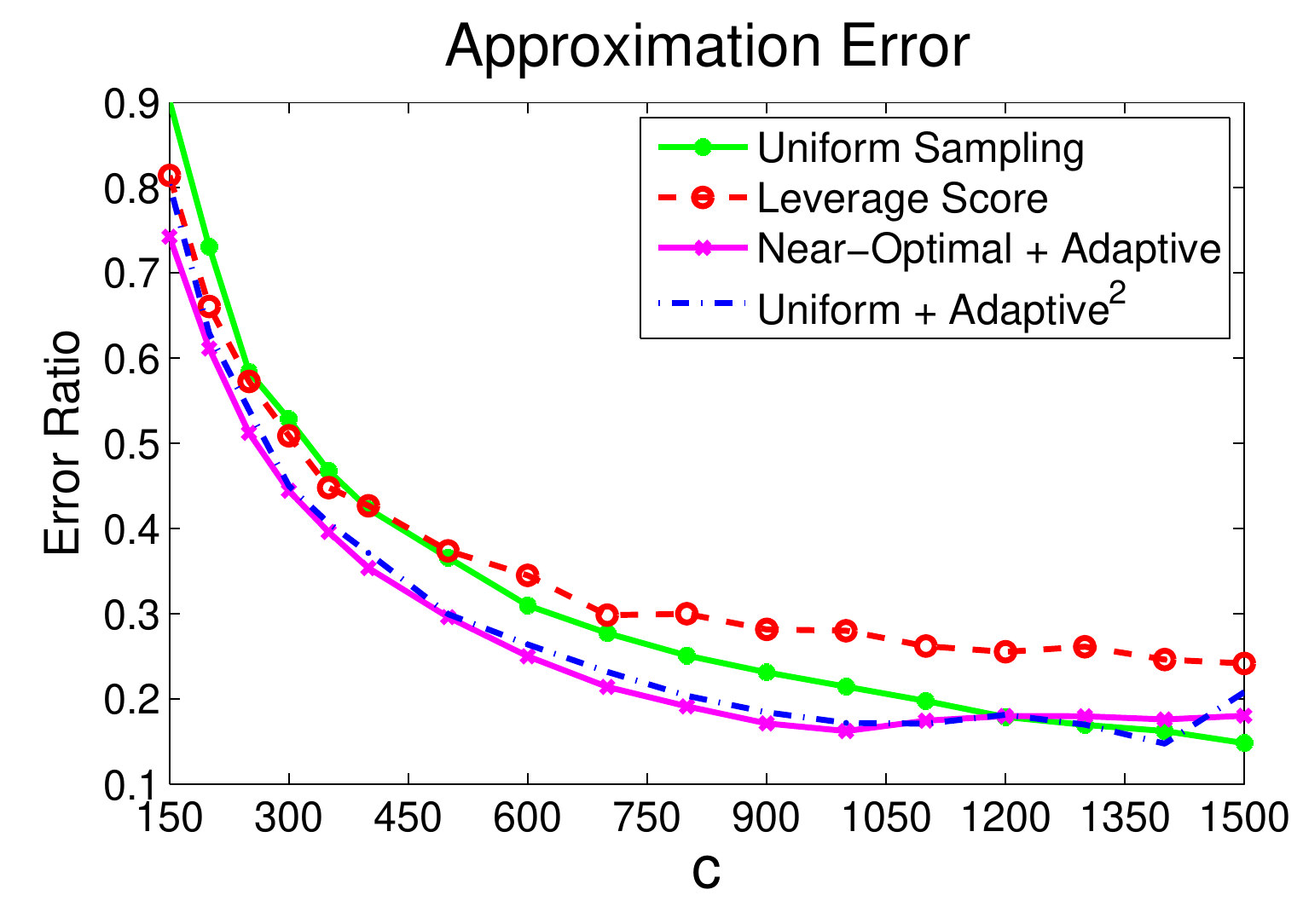}}
\end{center}
   \caption{Results on the RBF kernel of the Abalone dataset.
            Here the matrix coherence of the kernel matrix is $\mu_{10} = 3.28$, $\mu_{20} = 3.02$, and $\mu_{50} = 2.64$.}
\label{fig:abalone}
\end{figure*}

\begin{figure*}[!ht]
\subfigtopskip = 0pt
\begin{center}
\centering
\includegraphics[width=56mm,height=33mm]{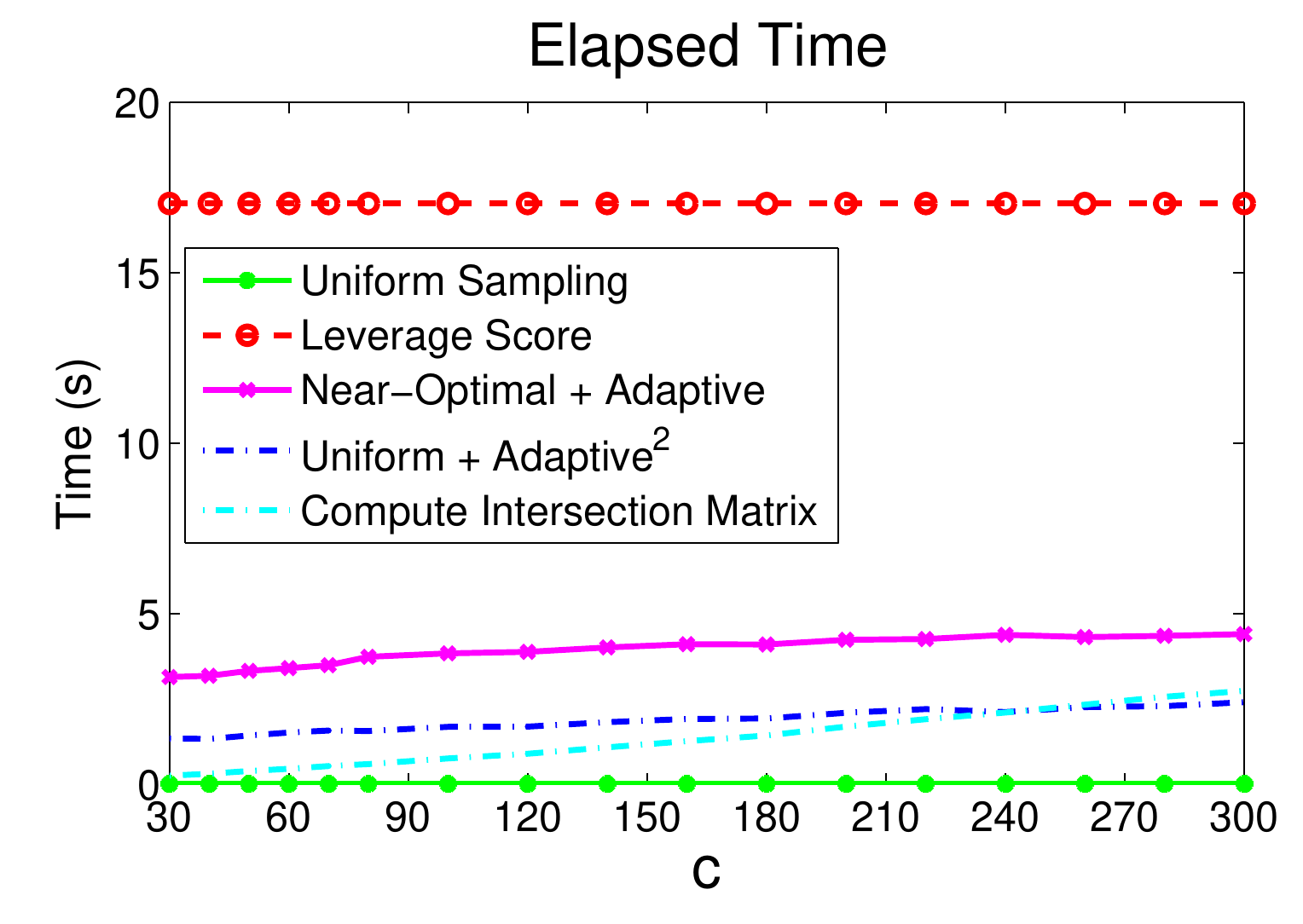}
\includegraphics[width=56mm,height=33mm]{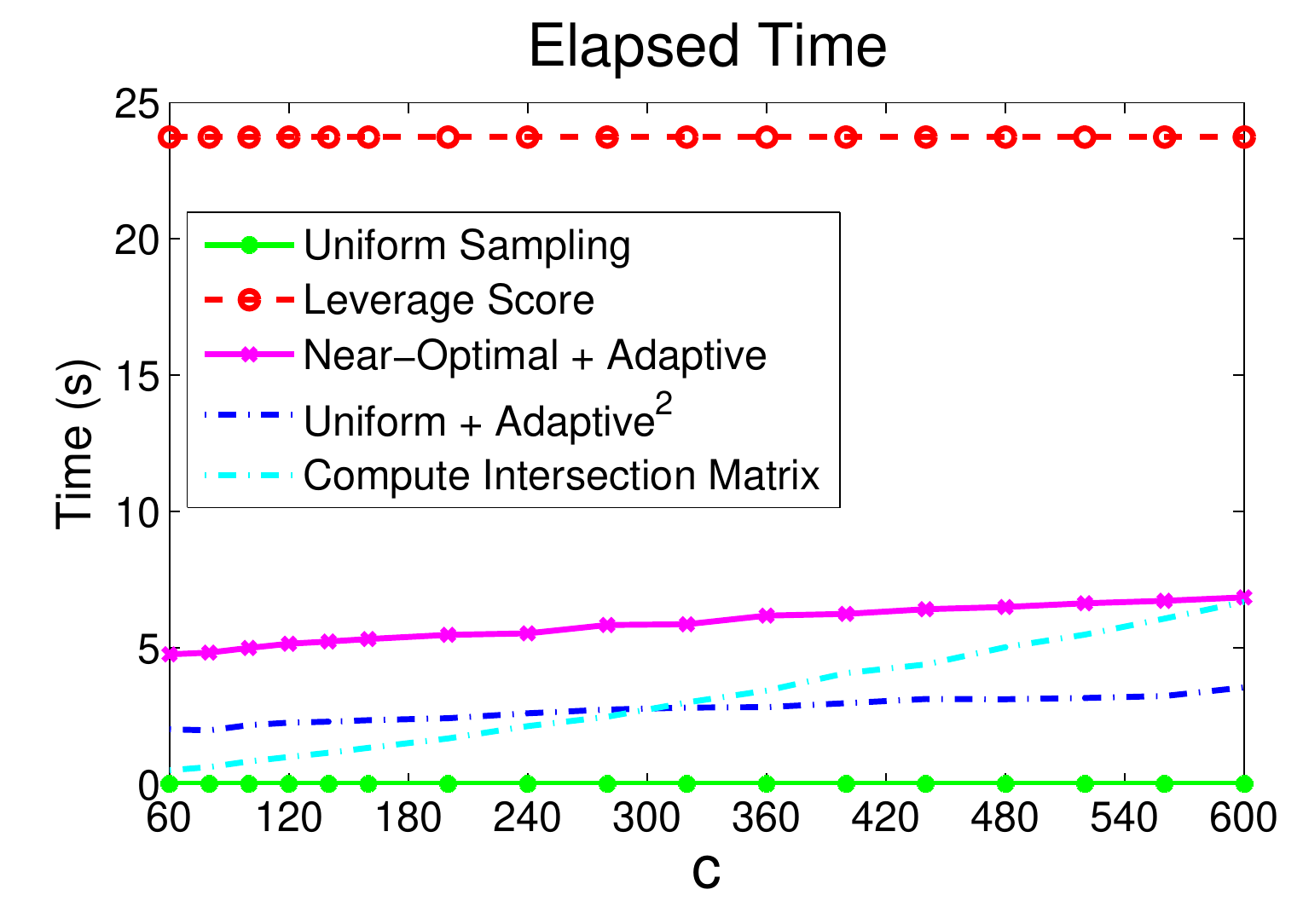}
\includegraphics[width=56mm,height=33mm]{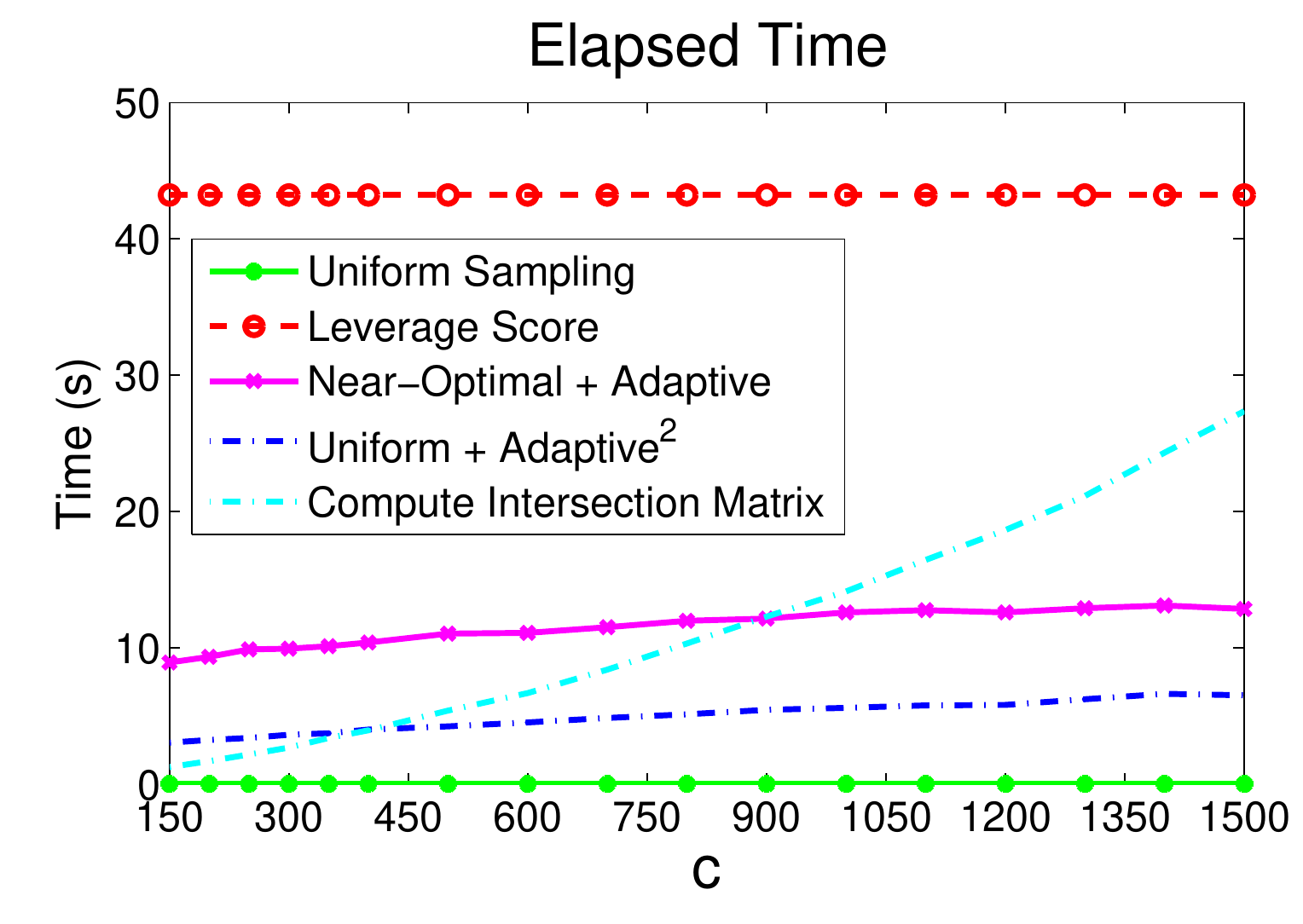}\\
\subfigure[$k = 10$]{\includegraphics[width=56mm,height=33mm]{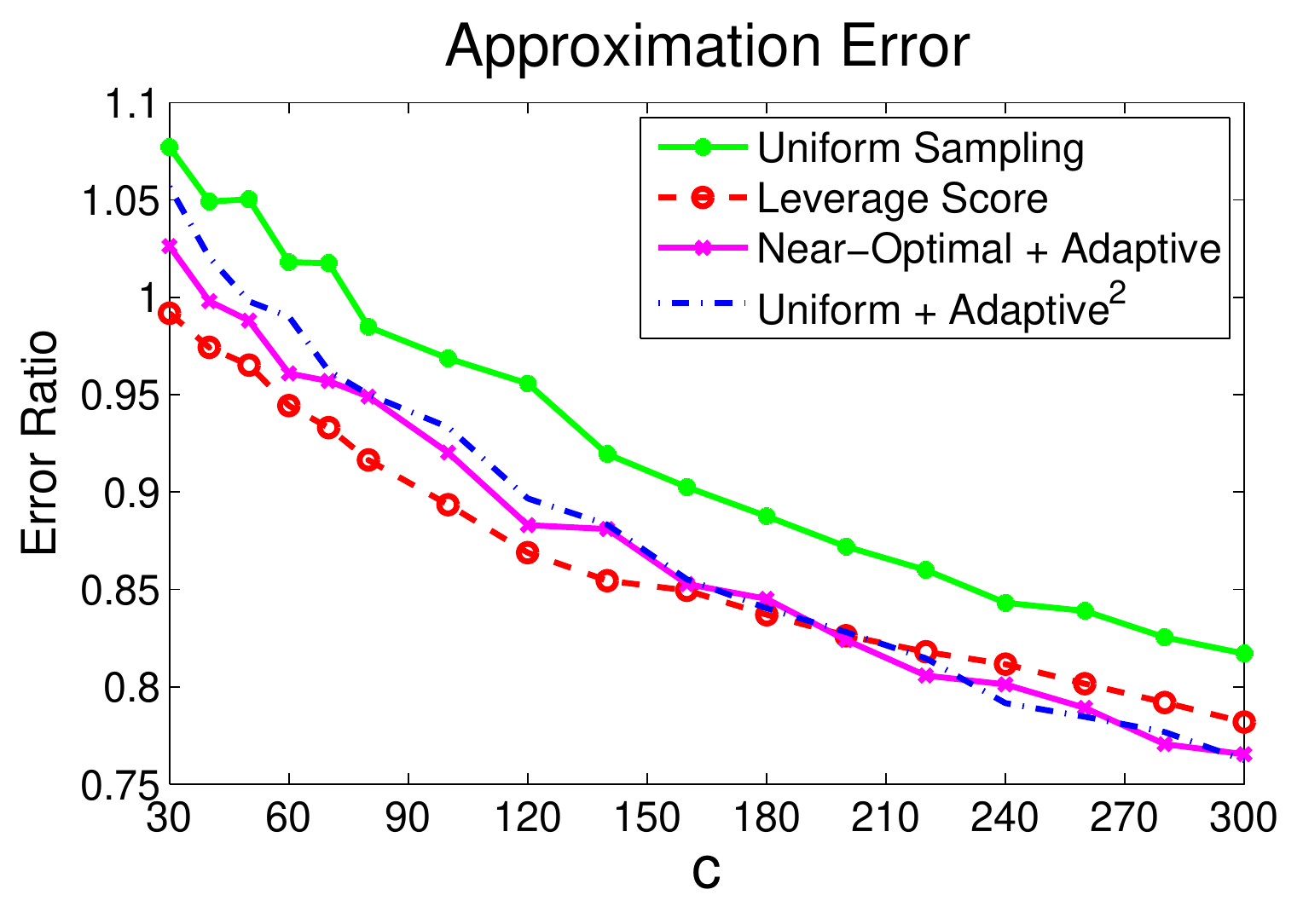}}
\subfigure[$k = 20$]{\includegraphics[width=56mm,height=33mm]{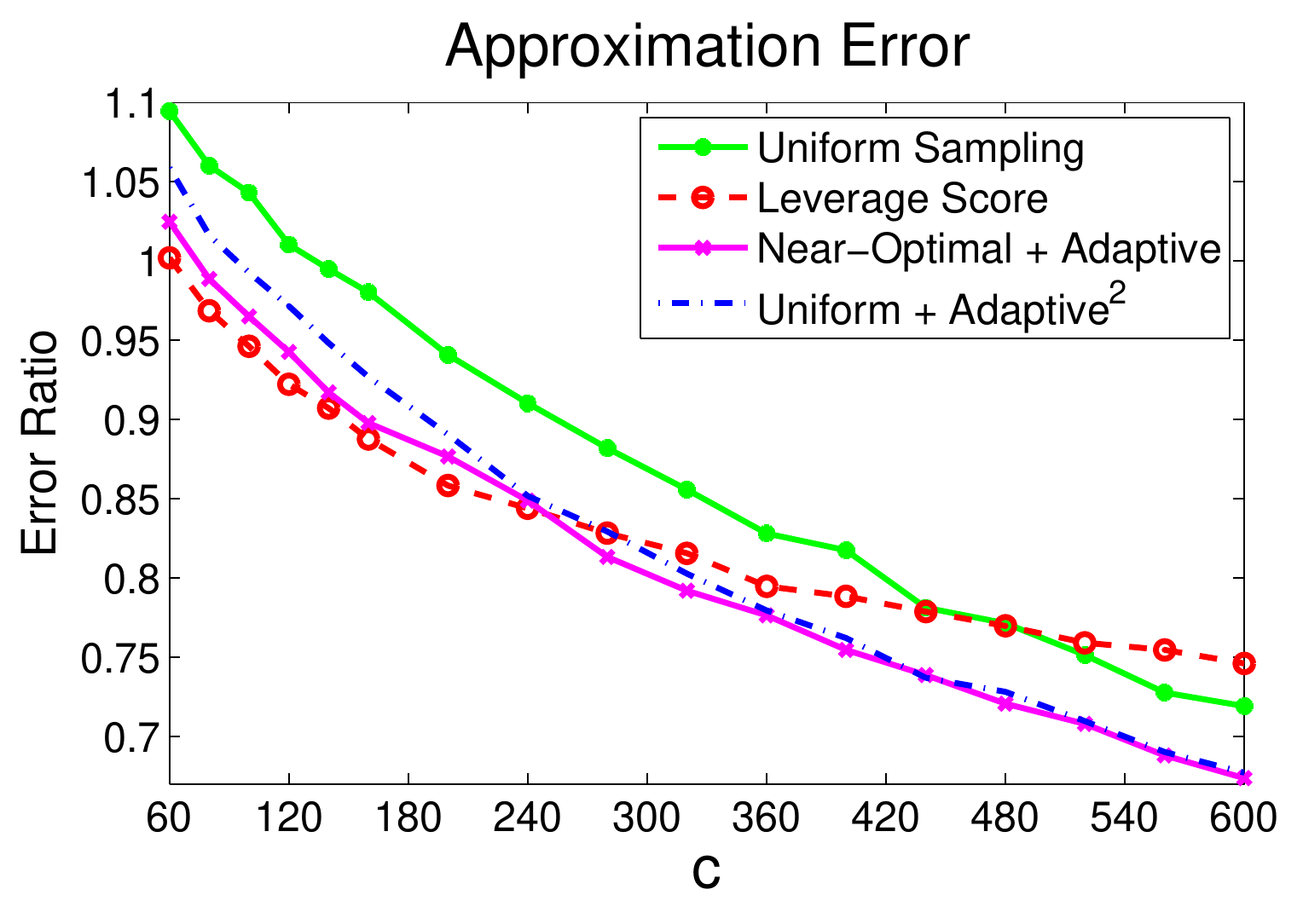}}
\subfigure[$k = 50$]{\includegraphics[width=56mm,height=33mm]{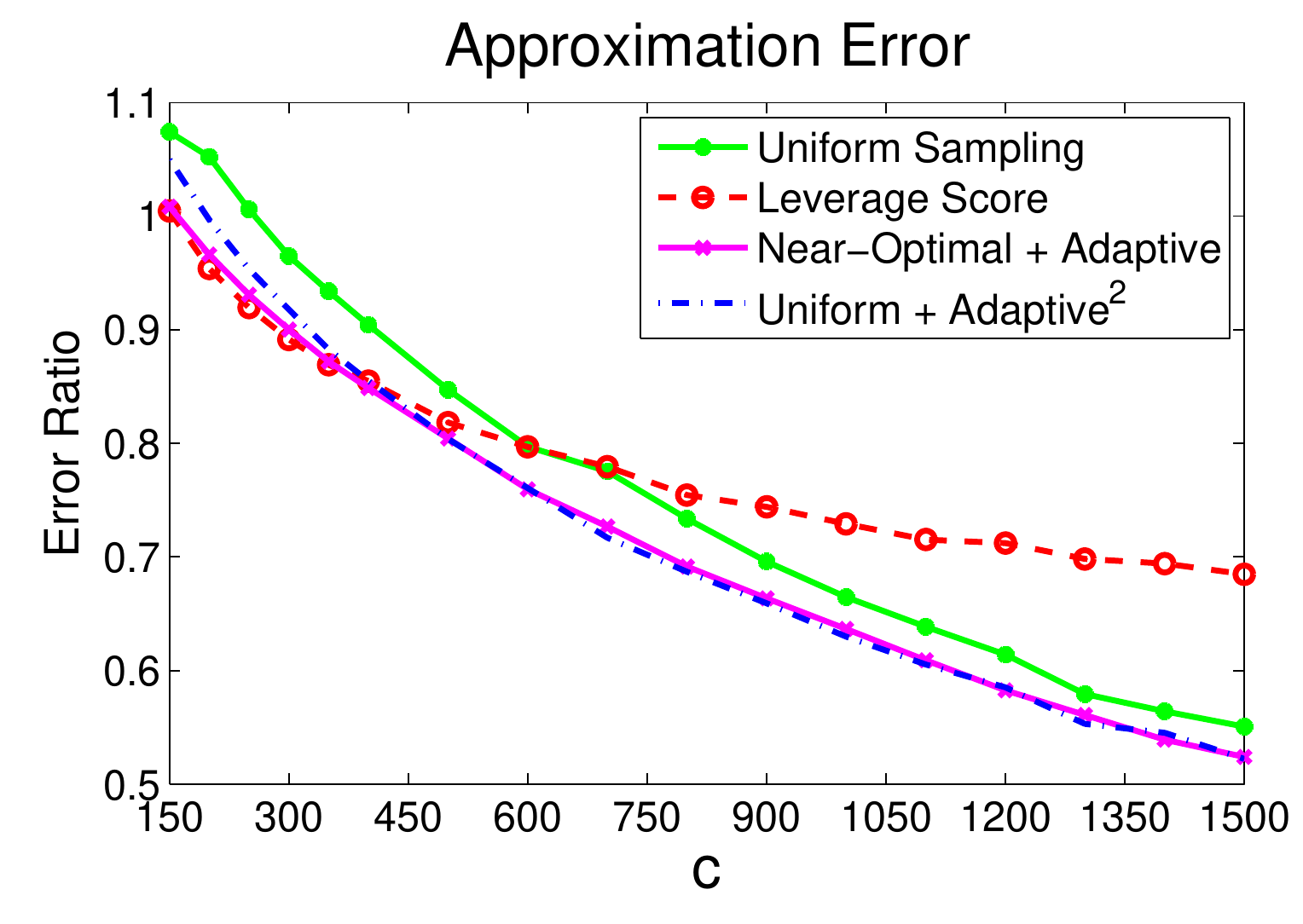}}
\end{center}
   \caption{Results on the RBF kernel of the Wine Quality dataset.
            Here the matrix coherence of the kernel matrix is $\mu_{10} = 16.17$, $\mu_{20} = 12.13$, and $\mu_{50} = 9.30$.}
\label{fig:wine}
\end{figure*}

\begin{table}[t]\setlength{\tabcolsep}{0.3pt}
\caption{A summary of the datasets for the \nystrom approximation.}
\label{tab:datasets_nystrom}
\begin{center}
\begin{footnotesize}
\begin{tabular}{c c c c c}
\hline
	{\bf Dataset}	&  {\bf ~\#Instance~}  &	{\bf ~\#Attribute~}   & {\bf ~Source}	\\
\hline
    Letters         &   $15,000$  &   $16$  &  \cite{michie1994machine} \\
    Abalone         &   $4,177$  &   $8$   &  \cite{uci2010} \\
    Wine Quality    &   $4,898$  &   $12$  &  \cite{cortez2009modeling} \\
\hline
\end{tabular}
\end{footnotesize}
\end{center}
\end{table}

\section{Experiments} \label{sec:experiments}

In this section we empirically evaluate our two algorithms proposed in Section~\ref{sec:efficient} and \ref{sec:intersection}.
In Section~\ref{sec:experiments:sampling} we compare the sampling algorithms for the modified \nystrom method in terms of approximation error and time expense.
In Section~\ref{sec:experiments:intersection} we illustrate the effect of our algorithm for computing the intersection matrix $\U=\C^\dag \A (\C^\dag)^T$.

We implement all of the compared algorithms in MATLAB and conduct
experiments on a workstation with Intel Xeon $2.40$GHz CPUs,
$24$GB RAM, and $64$bit Windows Server 2008 system.
To compare the running time, all the computations are carried out in a single thread in MATLAB.

\subsection{Comparisons among the Sampling Algorithms} \label{sec:experiments:sampling}

We mainly compare our uniform+adaptive$^2$ algorithm (Algorithm~\ref{alg:efficient}) with the near-optimal+adaptive algorithm~\citep{wang2013improving};
the two algorithms are the only provable algorithms for the modified \nystrom method.
We also employ the uniform sampling and the leverage-score based sampling \citep{drineas2008cur,gittens2013revisiting}
as baselines (they are widely used but not provable for the modified \nystrom method).
For all of the four algorithms, columns are sampled without replacement.

The experiment settings follows \cite{wang2013improving}.
We report the approximation error and running time of each algorithm on each dataset.
The approximation error is defined by
\begin{equation}  
\textrm{Approximation Error} \;=\; \frac{\| \A - \C \U \C^T\|_F}{\|\A - \A_k\|_F} , \nonumber
\end{equation}
where $k$ is a fixed target rank and $\U$ is the intersection matrix.

We test the algorithms on three datasets summarized in Table~\ref{tab:datasets_nystrom}.
For each dataset we generate an RBF kernel matrix $\A$ with
$a_{i j} =\exp \big( {-\frac{1}{2 \sigma^2}\|\x_i {-} \x_j \|_2^2} \big)$,
where $\x_i$ and $\x_j$ are data instances and $\sigma$ is the parameter defining the scale of the kernel.
We set $\sigma = 0.2$ in our experiments.
For each dataset we fix a target rank $k=10$, $20$, or $50$, and vary $c$ in a very large range.
We run each algorithm for $20$ times and report the the minimum approximation error of the $20$ repeats.
We also report the average elapsed time of column selection and the computation of the $c\times c$ intersection matrix, respectively.
Here we report the {\it average} elapsed time rather than the total time of the $20$ repeats because the $20$ repeats can be performed in parallel.
The results are depicted in Figures \ref{fig:letter}, \ref{fig:abalone}, and \ref{fig:wine}.

The empirical results in the figures show that our uniform+adaptive$^2$ algorithm achieves accuracy
comparable with the state-of-the-art algorithm---the near-optimal+adaptive algorithm of \cite{wang2013improving}.
Especially, when $c$ is large, those two algorithms have virtually the same accuracy,
which is in accordance with our analysis in the last paragraph of Section~\ref{sec:efficient}:
large $c$ implies small error term $\epsilon$, and the error bounds of the two algorithms coincide when $\epsilon$ is small.
We can also see that our uniform+adaptive$^2$ algorithm works nearly as good as the near-optimal+adaptive algorithm
when the matrix coherence $\mu_k$ is small (e.g. Figure~\ref{fig:abalone});
when the matrix coherence is large (e.g. Figure~\ref{fig:letter}),
the error of our algorithm is a little worse than the near-optimal+adaptive algorithm.
Furthermore, our uniform+adaptive$^2$ algorithm is much more accurate than uniform sampling and the leverage-score based sampling in most cases.

As for the running time, we can see that our algorithm performs column selection very efficiently
and the elapsed time grows slowly in $c$.
By comparison, our algorithm is much more efficient than the other two nonuniform sampling algorithms.

\vspace{-2mm}

\subsection{Effect of the Fast Computation of the Intersection Matrix} \label{sec:experiments:intersection}

\vspace{-2mm}

To illustrate the effect of our algorithm for computing the intersection matrix $\U=\C^\dag \A (\C^\dag)^T$,
we generate a kernel matrix of the Letters Dataset \citep{michie1994machine} which has $15,000$ instances and 16 attributes.
We first generate a dense RBF kernel matrix with scale parameter $\sigma=0.2$,
and then obtain a sparse symmetric matrix by by truncating the entries with small magnitude such that $1\%$ entries are nonzero.
We illustrate in Figure~\ref{fig:intersection_time} the speedup induced by our algorithm.
In both cases, our algorithm is faster than the naive approach, and the speedup is particularly significant when $\A$ is sparse.

\section{Theoretical Analysis for the Modified \nystrom Method} \label{sec:error_analysis}

\vspace{-2mm}

In Section~\ref{sec:exact} we show that the modified \nystrom approximation is exact when $\A$ is low-rank.
In Section~\ref{sec:lower} we provide a lower error bound of the modified \nystrom method.

\subsection{Theoretical Justifications} \label{sec:exact}

\vspace{-2mm}
\cite{kumar2009sampling,talwalkar2010matrix} showed that the standard \nystrom method is exact when $\rk(\W) = \rk(\A)$.
We show in Theorem~\ref{thm:connection} a similar result for the modified \nystrom approximations.

\begin{theorem}\label{thm:connection}
For a symmetric matrix $\A$ defined in (\ref{eq:definition_A}),
the following three statements are equivalent:
(i) $\rk(\W)=\rk(\A)$,
(ii) $\A = \C \W^\dag \C^T$,
(iii) $\A = \C \C^\dag \A (\C^\dag)^T \C^T$.
\end{theorem}

Theorem~\ref{thm:connection} shows that the standard and modified \nystrom methods are equivalent when $\rk(\W) = \rk(\A)$.
However, it holds in general that $\rk(\A) \gg c \geq \rk(\W)$, where the two models are not equivalent.

Furthermore,  $\U^{\textrm{mod}} = \C^\dag \A (\C^{\dag})^T$ is the minimizer of the following minimization problem
\[
\min_{\U} \; \|\A - \C \U \C^T\|_F ,
\]
so we have that
\[
\big\|\A - \C \big(\C^\dag \A (\C^{\dag})^T \big) \C \big\|_F \; \leq \; \big\| \A - \C \W^\dag \C \big\|_F.
\]
This shows that
in general the modified \nystrom method is more accurate than the standard \nystrom method.

\begin{figure}[t]
\begin{center}
\centering
\subfigure[{Dense RBF kernel matrix.}]{\includegraphics[width=77mm,height=38mm]{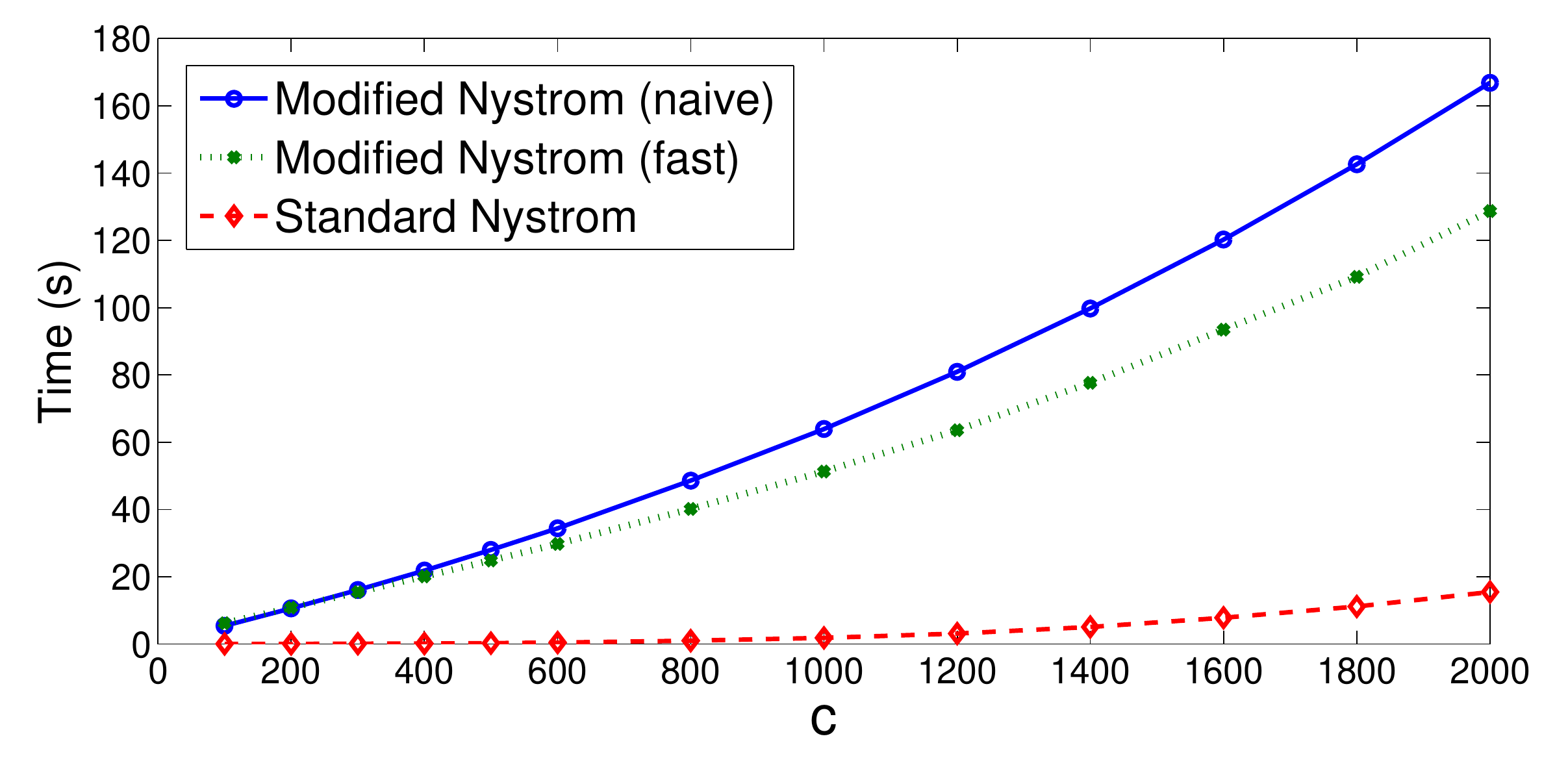}}
\subfigure[{Sparse RBF kernel matrix with $1\%$ nonzero entries.}]{\includegraphics[width=77mm,height=38mm]{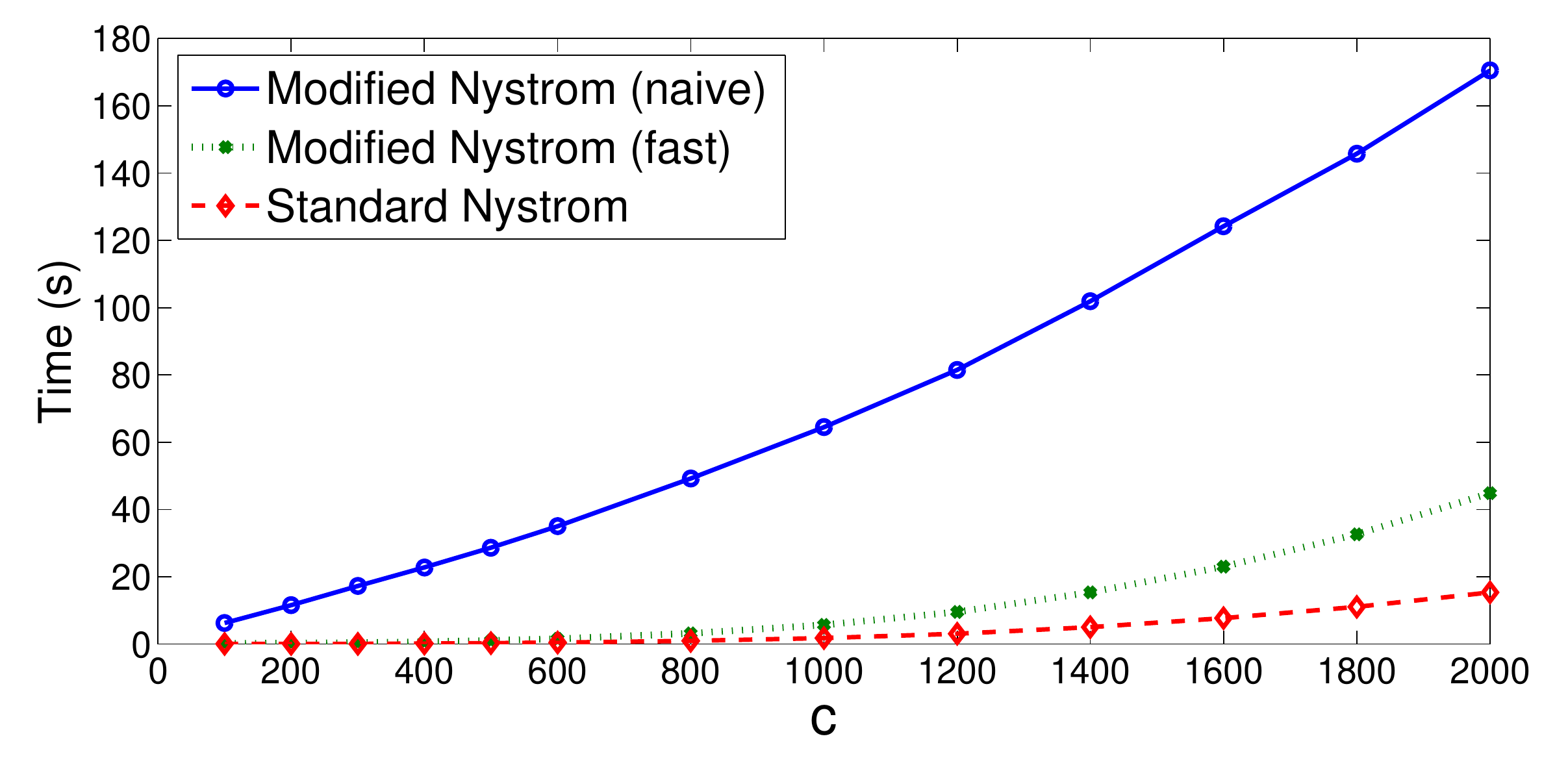}}
\end{center}
   \caption{Effect of our fast computation of the intersection matrix.
            The two matrices are both of size $15,000 \times 15,000$,
            and we sample $c$ columns uniformly to compute the intersection matrix $\U = \C^\dag \A (\C^\dag)^T$ (the modified Nystr\"om)
            and $\U= \W^\dag$ (the standard Nystr\"om).
            The time for computing $\U$ is plotted in the figures.}
\label{fig:intersection_time}
\end{figure}

\vspace{-2mm}

\subsection{Lower Error Bound of the Modified \nystrom Method} \label{sec:lower}

We establish in Theorem~\ref{thm:lower} a lower error bound of the modified \nystrom method.
Theorem~\ref{thm:lower} shows that whatever a column sampling algorithm is used to construct the modified \nystrom approximation,
at least $c \geq 2 k \epsilon^{-1}$ columns must be chosen to attain the $1+\epsilon$ bound.

\begin{theorem}[Lower Error Bound of the Modified \nystrom Method] \label{thm:lower}
Whatever a column sampling algorithm is used,
there exists an $m\times m$ SPSD matrix $\A$ such that the error incurred by the modified \nystrom method obeys:
\begin{eqnarray}
\big\| \A - \C \U \C^T \big\|_F^2 & \geq & \frac{m-c}{m-k} \Big( 1+ \frac{2 k}{c} \Big) \|\A - \A_k\|_F^2 .\nonumber
\end{eqnarray}
Here $k$ is an arbitrary target rank, $c$ is the number of selected columns,
and $\U = \C^\dag \A (\C^\dag)^T$.
\end{theorem}

\cite{boutsidis2011NOCshort} established a lower error bound for the column selection problem,
and the lower error bound is tight because it is attained by the optimal column selection algorithm of \cite{Guruswami2012optimal}.
\cite{boutsidis2011NOCshort} showed that whatever column sampling algorithm is used,
there exists an $m\times n$ matrix $\A$ such that the error incurred by the projection of $\A$ onto the column space of $\C$ is lower bounded by
\begin{eqnarray}\label{eq:lower_column_selection}
\big\| \A - \C \C^\dag \A \big\|_F^2 & \geq & \frac{n-c}{n-k} \Big( 1+ \frac{k}{c} \Big) \|\A - \A_k\|_F^2 ,
\end{eqnarray}
where $k$ is an arbitrary target rank, $c$ is the number of selected columns.

Interestingly, the modified \nystrom approximation is the projection of $\A$ onto the column space of $\C$ and the row space of $\C^T$ simultaneously,
so there is a strong resemblance between the modified \nystrom approximation and the column selection problem.
As we see, the lower error bound of the modified \nystrom approximation in Theorem~\ref{thm:lower}
differs from (\ref{eq:lower_column_selection}) only by a factor of $2$.
So it is a reasonable conjecture that the lower bound in Theorem~\ref{thm:lower} is tight,
as well a the lower bound of the column selection problem in (\ref{eq:lower_column_selection}).
We leave it as an open problem.

\section{Conclusions and Future Work}

\vspace{-2mm}

In this paper we have proposed two algorithms to make the modified \nystrom method more practical.
First, we have proposed a column selection algorithm called {\it uniform+adaptive$^2$} and provided an relative-error bound for the algorithm.
The algorithm is highly efficient and effective and very easy to implement.
The error bound of the algorithm is nearly as strong as that of the state-of-the-art
algorithm---the near-optimal+adaptive algorithm---which is complicated.
The experimental results have shown that our uniform+adaptive$^2$ algorithm is more efficient than the near-optimal+adaptive algorithm,
while their accuracies are comparable.
Second, we have devised an algorithm for computing the intersection matrix of the modified \nystrom approximation;
under certain conditions, our algorithm can significantly improve the time complexity.
The speedup induced by this algorithm has also been verified empirically.

Furthermore, we have proved that the modified \nystrom approximation can be exact when the original matrix is low-rank.
We have also established a lower error bound for the modified \nystrom method:
at least $c \geq 2 k \epsilon^{-1}$ columns must be chosen to attain the $1+\epsilon$ bound.
We have conjectured this lower error bound to be tight.
Notice that the best known algorithm for the modified \nystrom method requires at most $c= k \epsilon^{-2}$ columns to attain the $1+\epsilon$ bound,
so there is a gap between the lower and upper error bounds.
It remains an open problem that if there exists an algorithm attaining the lower error bound.

\section*{Acknowledgement}
\vspace{-2mm}

This work has been supported in part by the Natural Science Foundation of China (No. 61070239),
Microsoft Research Asia Fellowship 2013,
and the Scholarship Award for Excellent Doctoral Student granted by Chinese Ministry of Education.

\begin{small}
\bibliography{wangzhang2014efficient}
\bibliographystyle{Chicago}
\end{small}

\appendix

\section{Proof of Theorem~\ref{thm:efficient}} \label{sec:app:thm:efficient}

The error analysis for the uniform+adaptive$^2$ algorithm relies on Lemma~\ref{lem:uniform_column},
which guarantees the error incurred by its uniform sampling step.
The proof of Lemma~\ref{lem:uniform_column} essentially follows \cite{gittens2011spectral}.
We prove Lemma~\ref{lem:uniform_column} using probability inequalities and
some techniques of \cite{boutsidis2011NOCshort,gittens2011spectral,gittens2013revisiting,tropp2012user};
the proof is in Appendix~\ref{sec:lem:uniform_column}.

\begin{lemma}[Uniform Column Sampling] \label{lem:uniform_column}
Given an $m\times n$ matrix $\A$ and a target rank $k$, let $\mu_k$ denote the matrix coherence of $\A$.
By sampling
\[
c \;=\; \frac{\mu_k k \log (k / \delta)}{\theta \log \theta - \theta + 1},
\]
columns uniformly without replacement to construct $\C$,
the following inequality
\begin{eqnarray}
\big\|\A - \PM_{\C,k}  \A \big\|_F^2
&\leq & \big( 1 +  \delta^{-1} \theta^{-1}  \big) \big\| \A - \A_k \big\|_F^2 \textrm{.} \nonumber
\end{eqnarray}
holds with probability at least $1- 2 \delta$.
Here $\delta \in (0, 0.5)$ and $\theta \in (0, 1)$ are arbitrary real numbers.
\end{lemma}

The error analysis for the two adaptive sampling steps of the uniform+adaptive$^2$ algorithm relies on Lemma~\ref{lemma:wang2013improving},
which follows immediately from \cite[Corollary 7 and Section 4.5]{wang2013improving}.

\begin{lemma} \label{lemma:wang2013improving}
Given an $m\times m$ symmetric matrix $\A$ and a target rank $k$,
we let $\C_1$ contain the $c_1$ columns of $\A$ selected by a column sampling algorithm such that
the following inequality holds:
\begin{eqnarray}
\big\|\A - \PM_{\C_1}  \A \big\|_F^2
&\leq & f \big\| \A - \A_k \big\|_F^2 . \nonumber
\end{eqnarray}
Then we select $c_2 = k f \epsilon^{-1}$ columns to construct $\C_2$
and $c_3 = (c_1 + c_2) \epsilon^{-1}$ columns to construct $\C_3$,
both using the adaptive sampling according to the residual $\B_1 = \A - \PM_{\C_1} \A$ and $\B_2 = \A - \PM_{[\C_1,\C_2]} \A$, respectively.
Let $\C = [\C_1 , \C_2 , \C_3]$,
we have that
\[
\PB \Bigg\{ \frac{\big\| \A - \C \big( \C^\dag \A (\C^\dag)^T \big) \C^T \big\|_F}{\big\| \A - \A_k \big\|_F} \: \geq \: 1+ s \epsilon \Bigg\}
\; \leq \frac{1+\epsilon}{1+ s \epsilon} ,
\]
where $s$ is an arbitrary constant greater than $1$.
\end{lemma}

Finally Theorem~\ref{thm:efficient} is proved by combining Lemma~\ref{lem:uniform_column} and Lemma~\ref{lemma:wang2013improving}.
The proof is in Appendix~\ref{sec:thm:efficient}.

\subsection{Proof of Lemma~\ref{lem:uniform_column}} \label{sec:lem:uniform_column}

\begin{proof}
We use uniform column sampling to select $c$ column of $\A$ to construct $\C = \A \S$.
Here the ${n\times c}$ random matrix $\S$ has one entry equal to one and the rest equal to zero in each column,
and at most one nonzero entry in each row,
and $\S$ is uniformly distributed among $(^n_c)$ such kind of matrices.
Applying Lemma~7 of \cite{boutsidis2011NOCshort}, we get
\begin{align} \label{eq:thm:uniform_column:1}
&\!\!\!\!\!\! \big\|\A - \PM_{\C,k}  \A \big\|_F^2 \nonumber \\
&  \leq \big\| \A - \A_k \big\|_F^2
 + \big\| (\A - \A_k) \S \big\|_F^2 \: \big\| (\V_{\A,k}^T \S)^\dag \big\|_2^2 .
\end{align}
Now we bound $\big\| (\A - \A_k) \S \big\|_2^2$ and
$\big\| (\V_{\A,k}^T \S)^\dag \big\|_2^2$ respectively using the techniques of \cite{gittens2011spectral,gittens2013revisiting,tropp2012user}.

Let $\IM \subset [n]$ be a random index set corresponding to $\S$.
The support of $\IM$ is uniformly distributing among all the index sets in $2^{[n]}$ with cardinality $c$.
According to \cite{gittens2013revisiting},
the expectation of $\big\|(\A - \A_k) \S \big\|_F^2$ can be written as
\begin{align}
&\!\!\!\! \EB \big\|(\A - \A_k) \S \big\|_F^2
\;=\; \EB \big\| (\A - \A_k)_{\IM} \big\|_F^2 \nonumber \\
&\qquad\qquad =\;  c \EB \big\| (\A - \A_k)_i \big\|_F^2
\;=\;  \frac{c}{n} \big\| \A-\A_k \big\|_F^2  \textrm{.} \nonumber
\end{align}
Applying Markov's inequality, we have that
\begin{align}
&\!\!\!\!\!\!\!\!\!\!\!\!\!\!\!\! \PB \bigg\{ \big\|(\A - \A_k) \S \big\|_F^2 \geq \frac{c}{n \delta} \big\| \A-\A_k \big\|_F^2 \bigg\} \nonumber \\
&\qquad\qquad \leq \;  \frac{\EB \big\| (\A - \A_k) \S \big\|_F^2}{\frac{c}{n \delta} \big\| \A-\A_k \big\|_F^2 }
    \;=\; \delta \textrm{.} \label{eq:thm:uniform_column:2}
\end{align}
Here $\delta \in (0, 0.5)$ is a real number defined later.

Now we establish the bound for $\EB \big\| \Ome_2^\dag \big\|_2^2$ as follows.
Let $\lambda_i (\X)$ be the $i$-th largest eigenvalue of $\X$.
Following the proof of Lemma~1 of \cite{gittens2011spectral}, we have
\begin{align}\label{eq:thm:uniform_column:3}
& \!\!\!  \big\| ( \V_{\A,k}^T \S )^\dag \big\|_2^2
\;=\; \lambda^{-1}_{k} \Big( \V_{\A,k}^T \S \S^T \V_{\A,k} \Big)  \nonumber \\
&\qquad\qquad =\; \lambda^{-1}_{k} \Big( \sum_{i=1}^c \X_i \Big)
\;\leq \; \lambda^{-1}_{\min} \Big( \sum_{i=1}^c \X_i \Big) ,
\end{align}
where the random matrices $\X_1 , \cdots , \X_c$ are chosen uniformly at random from the set
$\Big\{ \big( \V_{\A,k}^T \big)_i \big(\V_{\A,k}^T \big)_i^T \Big\}_{i=1}^{n}$ without replacement.
The random matrices are of size $k\times k$.
We accordingly define
\[
R = \max_{i} \lambda_{\max} (\X_i) = \max_{i} \big\| \big( \V_{\A,k}^T \big)_i \big\|_2^2 = \frac{k}{n} \mu_k \textrm{,}
\]
where $\mu_k$ is the matrix coherence of $\A$, and define
\begin{eqnarray}
\beta_{\min}
& = & c  \lambda_{\min}\big( \EB \X_1 \big) \nonumber\\
& = & \lambda_{\min}\Big( \frac{c}{n}  \V_{\A,k}^T \V_{\A,k} \Big)
    \;=\; \frac{c}{n} \textrm{.}\nonumber
\end{eqnarray}
Then we apply Lemma~\ref{lemma:matrix_chernoff} and obtained the following inequality:
\begin{eqnarray} \label{eq:thm:uniform_column:lambda_min}
\PB\bigg[ \lambda_{\min} \Big(  \sum_{i=1}^c \X_i \Big) \leq \frac{\theta c}{n}\bigg]
& \leq & k \bigg[ \frac{e^{\theta-1}}{\theta^{\theta}} \bigg]^{\frac{c}{k \mu_k}}
\;\triangleq \; \delta \textrm{,}
\end{eqnarray}
where $\theta\in (0, 1]$ is a real number,
and it follows that
\[
c \;=\; \frac{\mu_k k \log (k / \delta)}{\theta \log \theta - \theta + 1}.
\]
Applying (\ref{eq:thm:uniform_column:3}) and (\ref{eq:thm:uniform_column:lambda_min}), we have
\begin{eqnarray} \label{eq:thm:uniform_column:4}
\PB \Big\{ \big\| ( \V_{\A,k}^T \S )^\dag \big\|_2^2 \geq \frac{n}{\theta c} \Big\}
&\leq& \delta \textrm{.}
\end{eqnarray}

Combining (\ref{eq:thm:uniform_column:2}) and (\ref{eq:thm:uniform_column:4}) and applying the union bound,
we have the following inequality:
\begin{align}
&\!\!\!\!\!\!\!\!\!\!\!\!\! \PB \bigg\{ \big\|(\A - \A_k) \S \big\|_F^2 \geq \frac{c }{n \delta} \big\| \A-\A_k \big\|_F^2 \nonumber\\
& \qquad \qquad \textrm{ or } \quad  \big\| ( \V_{\A,k}^T \S )^\dag \big\|_2^2 \geq \frac{n}{\theta c}  \bigg\}
     \; \leq \;  2 \delta \textrm{.} \label{eq:thm:uniform_column:5}
\end{align}

Finally, from (\ref{eq:thm:uniform_column:1}) and (\ref{eq:thm:uniform_column:5})
we have that the inequality
\begin{eqnarray}
 \big\|\A - \PM_{\C,k}  \A \big\|_F^2
& \leq & \big( 1 +  \delta^{-1} \theta^{-1}  \big) \big\| \A - \A_k \big\|_F^2 \nonumber
\end{eqnarray}
holds with probability at least $1- 2 \delta$,
by which the lemma follows.
\end{proof}

\begin{lemma}[Theorem 2.2 of \cite{tropp2012user}] \label{lemma:matrix_chernoff}
We are given $l$ independent random $d\times d$ SPSD matrices $\X_1 ,\cdots , \X_l$  with the property
\[
\lambda_{\max} (\X_i) \leq R \quad \textrm{ for } \; i = 1, \cdots , l \textrm{.}
\]
We define $\Y = \sum_{i=1}^l \X_i$ and $\beta_{\min} = l \lambda_{\min}\big( \EB \X_1$\big).
Then for any $\theta \in (0, 1]$, the following inequality holds:
\[
\PB \Big\{ \lambda_{\min} (\Y) \leq \theta \beta_{\min}  \Big\}
\;\leq\;
d \bigg[ \frac{e^{\theta-1}}{\theta^{\theta}} \bigg]^{\frac{\beta_{\min}}{R}}.
\]
\end{lemma}

\subsection{Proof of the Theorem} \label{sec:thm:efficient}

\begin{proof}
The matrix $\C_1$ consists of $c_1$ columns selected by uniform sampling,
and $\C_2 \in \RB^{m\times c_2}$ and $\C_3 \in \RB^{m\times c_3}$ are constructed by adaptive sampling.
We set $\delta = 1/\sqrt{5}$ and $\theta =\sqrt{5}/4$ for Lemma~\ref{lem:uniform_column},
then we have
\begin{eqnarray}
f   &=& 1 + \delta^{-1} \theta^{-1} \; = \; 5 ,\nonumber \\
c_1 &=& \frac{\mu_k k \log (k / \delta)}{\theta \log \theta - \theta + 1} \; = \; 8.7 \mu_k k \log (\sqrt{5} k) .\nonumber
\end{eqnarray}
Then we set
\begin{eqnarray}
c_2 &=& k f \epsilon^{-1} \; = \; 5k \epsilon^{-1} ,\nonumber \\
c_3 &=& (c_1 + c_2) \epsilon^{-1} ,\nonumber
\end{eqnarray}
according to Lemma~\ref{lemma:wang2013improving}.
Letting $s > 1$ be an arbitrary constant, we have that
\begin{small}
\begin{align}
&\PB \Bigg\{ \frac{\big\| \A - \C \U \C^T \big\|_F}{\big\| \A - \A_k \big\|_F}  \leq  1+ s \epsilon \Bigg\} \nonumber \\
& \geq \; \PB \Bigg\{ \frac{\big\| \A - \C \U \C^T \big\|_F}{\big\| \A - \A_k \big\|_F}  \leq  1+ s \epsilon \;\: \Bigg| \;\: \frac{\big\|\A - \PM_{\C_1}  \A \big\|_F^2}{\big\| \A - \A_k \big\|_F^2} \leq f \Bigg\} \nonumber \\
& \qquad\qquad\qquad\qquad\qquad  \cdot \; \PB \Bigg\{ \frac{\big\|\A - \PM_{\C_1}  \A \big\|_F^2}{\big\| \A - \A_k \big\|_F^2} \leq f \Bigg\} \nonumber \\
& \geq \; \Big( 1- \frac{1+\epsilon}{1+s\epsilon} \Big) \Big(1 - 2\delta\Big) . \nonumber
\end{align}
\end{small}
where the last inequality follows from Lemma~\ref{lem:uniform_column} and Lemma~\ref{lemma:wang2013improving}.

Repeating the sampling procedure for $t$ times and letting $\C_{[i]}$ and $\U_{[i]}$ be the $i$-th sample,
we obtain an upper error bound on the failure probability:
\begin{align}
&\!\!\! \PB \Bigg\{ \min_{i \in [t]} \bigg\{ \frac{\big\| \A - \C_{[i]} \U_{[i]} \C_{[i]}^T \big\|_F}{\big\| \A - \A_k \big\|_F} \bigg\}
                    \; \geq \;  1+ s \epsilon \Bigg\} \nonumber \\
& \qquad\qquad \leq \; \bigg( 1 - \Big( 1- \frac{1+\epsilon}{1+s\epsilon} \Big) \Big(1 - 2\delta\Big) \bigg)^t \nonumber \\
& \qquad\qquad \;=\; \bigg( 1 + \frac{(s-1)(1-2\delta)}{\epsilon^{-1} + 1 + 2\delta (s-1)} \bigg)^{-t}
    \; \triangleq \; p . \nonumber
\end{align}
Taking logarithm of both sides of the equality and applying $\log (1+x) \approx x$ when $x$ is small,
we have
\begin{eqnarray}
t & = & \bigg[ \log\Big( 1 + \frac{(1-2\delta)(s-1)}{\epsilon^{-1} + 1 + 2\delta (s-1)} \Big) \bigg]^{-1} \log\frac{1}{p}  \nonumber \\
& \approx & \frac{\epsilon^{-1} + 1 + 2\delta (s-1)}{(1-2\delta)(s-1)} \log\frac{1}{p} .\nonumber
\end{eqnarray}
Setting $s=2$, we have that $t\approx (10\epsilon^{-1} + 18) \log (1/p)$.

Hence by sampling totally
\[
c \;=\; \big(1+\epsilon^{-1}\big) \big( 5 k \epsilon^{-1} + 8.7 \mu_k k \log( \sqrt{5} k ) \big)
\]
columns
and repeating the procedure for
\[
t \;\geq\; (10\epsilon^{-1} + 18) \log (1/p)
\]
times, the algorithm attains the upper error bound
\begin{eqnarray}
\big\|\A - \C \big( \C^\dag \A (\C^{\dag})^T \big) \C^T \big\|_F
\;\leq \; \big(1 + 2 \epsilon \big) \big\| \A - \A_k \big\|_F \nonumber
\end{eqnarray}
with probability at least $1-p$.
Substituting $2\epsilon$ by $\epsilon'$ yields the error bound in the theorem.

Time complexity and space complexity of Algorithm~\ref{alg:efficient} is calculated as follows.
The uniform sampling costs $\OM(m)$ time;
the first adaptive sampling round costs $\OM(m c_1^2) + \TimeMulti (m^2 c_1)$ time;
the second adaptive sampling round costs $\OM(m (c_1+c_2)^2) + \TimeMulti (m^2 (c_1+c_2))$ time;
computing the intersection matrix costs $\OM(m c^2) + \TimeMulti (m^2 c)$ time in general.
So the total time complexity is $\OM(m c^2) + \TimeMulti (m^2 c)$ without using Theorem~\ref{thm:intersection_matrix},
or $\OM(m (c_1+c_2)^2) +  \TimeMulti (m^2 c) $ using Theorem~\ref{thm:intersection_matrix}.
As for the space complexity, the Moore-Penrose inverse of an $m\times c$ matrix demands $\OM(mc)$ space,
and multiplying a $c \times m$ matrix $\C^\dag$ by an $m\times m$ matrix $\A$ costs $\OM(m c)$ space
by partition $\A$ into small blocks of size smaller than $m\times c$
and loading one block into RAM at a time to perform matrix multiplication.
\end{proof}

\section{Proof of Theorem~\ref{thm:intersection_matrix}} \label{sec:thm:intersection_matrix}

\begin{proof}
Let $\C \in \RB^{m\times c}$ consists of a subset of columns of $\A$.
By row permutation $\C$ can be expressed as
\[
\PP \C \; = \; \left[
                \begin{array}{c}
                  \W \\
                  \A_{2 1} \\
                \end{array}
              \right].
\]
Then according to Lemma~\ref{lem:pinv_partitioned_matrix}, the Moore-Penrose inverse of $\C$ can be written as
\[
\C^\dag = \W^{-1} \big( \I_c + \S^T \S \big)^{-1}
\left[
  \begin{array}{cc}
    \I_c & \S^T \\
  \end{array}
\right] \PP \textrm{,}
\]
where $\S = \A_{2 1} \W^{-1}$.
Then the intersection matrix of modified \nystrom approximation to $\A$ can be expressed as
\begin{align}
&\U \;=\; \C^\dag \A \big(\C^\dag\big)^T \nonumber \\
& \quad\: =\; \W^{-1} \big( \I_c + \S^T \S \big)^{-1}
    \left[
      \begin{array}{c c}
        \I_c & \S^T \\
      \end{array}
    \right]
    \PP \A \PP^T    \nonumber \\
& \qquad\qquad\qquad
    \left[
      \begin{array}{c}
        \I_c \\
        \S \\
      \end{array}
    \right]
    \big( \I_c + \S^T \S \big)^{-1} \W^{-1} \nonumber \\
& \quad\: =\; \W^{-1} \big( \I_c + \S^T \S \big)^{-1}
    \left[
      \begin{array}{c c}
        \I_c & \S^T \\
      \end{array}
    \right]\nonumber \\
& \qquad\qquad
    \left[
      \begin{array}{c c}
        \W & \A_{2 1}^T \\
        \A_{2 1} & \A_{2 2} \\
      \end{array}
    \right]
    \left[
      \begin{array}{c}
        \I_c \\
        \S \\
      \end{array}
    \right]
    \big( \I_c + \S^T \S \big)^{-1} \W^{-1} \nonumber \\
& \quad\: =\; \W^{-1} \big( \I_c + \S^T \S \big)^{-1}
    \Big( \W + \A_{2 1}^T\S + (\A_{2 1}^T\S)^T\nonumber \\
& \qquad\qquad\qquad   + \S^T \A_{2 2} \S \Big) \big( \I_c + \S^T \S \big)^{-1} \W^{-1} \nonumber \\
& \quad\: \triangleq \; \T_1 \big( \W + \T_2 + \T_2^T + \T_3 \big) \T_1^T \textrm{.} \nonumber 
\end{align}
Here the intermediate matrices are computed by
\begin{eqnarray*}
\T_0 &=& \A_{2 1}^T \A_{2 1} \textrm{,}  \\
\T_1 &=& \W^{-1} \big( \I_c + \S^T \S \big)^{-1}\\
     &=& \W^{-1} \Big( \I_c + \W^{-1} \T_0 \W^{-1} \Big)^{-1} \textrm{,}   \\
\T_2 &=& \A_{2 1}^T\S
    \;=\; \A_{2 1}^T \A_{2 1} \W^{-1}
    \;=\; \T_0 \W^{-1} \textrm{,}   \\
\T_3 &=& \S^T \A_{2 2} \S
    \;=\; \W^{-1} \Big(\A_{2 1}^T \A_{2 2} \A_{2 1} \Big)\W^{-1} \textrm{.}
\end{eqnarray*}
The matrix inverse operations are on $c\times c$ matrices which costs $\OM(c^3)$ time.
The matrix multiplication $\A_{2 1}^T \A_{2 2} \A_{2 1}$ requires time $T_{\mathrm{Multiply}} \big( (m-c)^2 c\big)$.
\end{proof}

\begin{lemma}[The Moore Penrose Inverse of Partitioned Matrices {\cite[Page~179]{adi2003inverse}}] \label{lem:pinv_partitioned_matrix}
Given a matrix $\X\in \RBmn$ of rank of at least $c$ which has a nonsingular $c\times c$ submatrix $\X_{1 1}$.
By rearrangement of columns and rows by permutation matrices $\PP$ and $\Q$,
the submatrix $\X_{1 1}$ can be bought to the top left corner of $\X$, that is,
\[
\PP \X \Q \; =\; \left[
             \begin{array}{cc}
               \X_{1 1} & \X_{1 2} \\
               \X_{2 1} & \X_{2 2} \\
             \end{array}
           \right].
\]
Then the Moore-Penrose inverse of $\X$ is
\begin{align}
&\X^\dag \; =\;
\Q \left[
     \begin{array}{c}
       \I_c \\
       \T^T \\
     \end{array}
   \right]
\big( \I_c + \T \T^T \big)^{-1} \X_{1 1}^{-1}  \nonumber\\
&\qquad\qquad\qquad \big( \I_c + \S \S^T \big)^{-1}
\left[
  \begin{array}{cc}
    \I_c & \S^T  \\
  \end{array}
\right] \PP ,\nonumber
\end{align}
where $\T = \X_{1 1}^{-1} \X_{1 2}$ and
$\S = \X_{2 1} \X_{1 1}^{-1}$.
\end{lemma}

\section{The Proof of Theorem~\ref{thm:connection}}

\begin{proof}
Suppose that $\rk(\W) = \rk(\A)$.
We have that $\rk(\W) = \rk(\C) = \rk(\A)$ because
\begin{equation} \label{eq:thm:connection:4}
\rk(\A) \;\geq\; \rk(\C) \;\geq\; \rk(\W) \textrm{.}
\end{equation}
Thus there exists a matrix $\X$ such that
\[
\left[
  \begin{array}{c}
    \A_{2 1}^T \\
    \A_{2 2} \\
  \end{array}
\right]
\; = \;
\C \X^T
\; = \;
\left[
  \begin{array}{c}
    \W \X^T \\
    \A_{2 1} \X^T \\
  \end{array}
\right] \textrm{,}
\]
and it follows that $\A_{2 1} = \X \W$
and $\A_{2 2} = \A_{2 1} \X^T = \X \W \X^T$.
Then we have that
\begin{eqnarray}
\A
& = & \left[
        \begin{array}{cc}
          \W & (\X \W)^T \\
          \X \W & \X \W \X^T \\
        \end{array}
      \right] \nonumber\\
& = &
    \left[
        \begin{array}{c}
          \I \\
          \X \\
        \end{array}
      \right]
      \W
      \left[
        \begin{array}{cc}
          \I & \X^T \\
        \end{array}
      \right] \textrm{,} \label{eq:thm:connection:1} \\
\C \W^\dag \C^T
& = &
    \left[
        \begin{array}{c}
          \W \\
          \X \W \\
        \end{array}
      \right]
      \W^\dag
      \left[
        \begin{array}{cc}
          \W & (\X \W)^T \\
        \end{array}
      \right] \nonumber\\
& = &
    \left[
        \begin{array}{c}
          \I \\
          \X \\
        \end{array}
      \right]
      \W
      \left[
        \begin{array}{cc}
          \I & \X^T \\
        \end{array}
      \right] \textrm{.} \label{eq:thm:connection:2}
\end{eqnarray}
Here the second equality in (\ref{eq:thm:connection:2})
follows from $\W \W^\dag \W = \W$.
We obtain that $\A = \C \W^\dag \C$.
Then we show that $\A = \C \C^\dag \A (\C^\dag)^T \C^T$.

Since $\C^\dag = (\C^T \C)^\dag \C^T$, we have that
\[
\C^\dag
\;=\;
\big( \W (\I + \X^T \X) \W \big)^\dag \W \: [\I \, , \, \X^T] \textrm{,}
\]
and thus
\begin{align}
&\C^\dag \A (\C^\dag)^T \W \nonumber\\
&=\;
\big( \W (\I + \X^T \X) \W \big)^\dag
\W (\I + \X^T \X) \Big[ \W (\I + \X^T \X) \nonumber\\
&\qquad\qquad \W \big( \W (\I + \X^T \X) \W \big)^\dag \W \Big]\nonumber\\
&=\;
\big( \W (\I + \X^T \X) \W \big)^\dag
\W (\I + \X^T \X) \W \textrm{,} \nonumber
\end{align}
where the second equality follows from Lemma~\ref{lem:thm:connection}
because $(\I + \X^T \X)$ is positive definite.
Similarly we have
\begin{align}
&\W \C^\dag \A (\C^\dag)^T \W \nonumber\\
&=\;
\W \big( \W (\I + \X^T \X) \W \big)^\dag
\W (\I + \X^T \X) \W
\;=\; \W\textrm{.} \nonumber
\end{align}
Thus we have
\begin{small}
\begin{eqnarray}
\C \C^\dag \A (\C^\dag)^T \C
&=&    \left[
        \begin{array}{c}
          \I \\
          \X \\
        \end{array}
      \right]
      \W \C^\dag \A (\C^\dag)^T \W
      \left[
        \begin{array}{cc}
          \I & \X^T \\
        \end{array}
      \right]  \nonumber \\
&=&    \left[
        \begin{array}{c}
          \I \\
          \X \\
        \end{array}
      \right]
      \W
      \left[
        \begin{array}{cc}
          \I & \X^T \\
        \end{array}
      \right] \textrm{.}  \label{eq:thm:connection:3}
\end{eqnarray}
\end{small}
It follows from Equations (\ref{eq:thm:connection:1}) (\ref{eq:thm:connection:2}) (\ref{eq:thm:connection:3})
that $\A = \C \W^\dag \C^T = \C \C^\dag \A (\C^\dag)^T \C^T$.

Conversely, when $\A = \C \W^\dag \C^T$, we have that $\rk(\A) \leq \rk(\W^\dag) = \rk(\W)$.
By applying (\ref{eq:thm:connection:4}) we have that $\rk(\A) = \rk(\W)$.

When $\A = \C \C^\dag \A (\C^\dag)^T \C^T$,
we have $\rk(\A) \leq \rk(\C)$.
Thus there exists a matrix $\X$ such that
\[
\left[
  \begin{array}{c}
    \A_{2 1}^T \\
    \A_{2 2} \\
  \end{array}
\right]
\; = \;
\C \X^T
\; = \;
\left[
  \begin{array}{c}
    \W \X^T \\
    \A_{2 1} \X^T \\
  \end{array}
\right] \textrm{,}
\]
and therefore $\A_{2 1} = \X \W$.
Then we have that
\[
\C \;=\;
\left[
  \begin{array}{c}
    \W \\
    \A_{2 1} \\
  \end{array}
\right]
\;=\;
\left[
  \begin{array}{c}
    \I \\
    \X \\
  \end{array}
\right]
\W \textrm{,}
\]
so $\rk(\C) \leq \rk(\W)$.
Apply (\ref{eq:thm:connection:4}) again we have $\rk(\A) = \rk(\W)$.
\end{proof}

\begin{lemma} \label{lem:thm:connection}
$\X^T \V \X \big( \X^T \V \X\big)^\dag \X^T = \X^T$ for any positive definite matrix $\V$.
\end{lemma}

\begin{proof}
Since the positive definite matrix $\V$ have a decomposition $\V= \B^T \B$ for some nonsingular matrix $\B$,
so we have
\begin{align}
& \X^T \V \X \big( \X^T \V \X\big)^\dag \X^T \nonumber \\
& = \; (\B \X)^T \Big(\B \X \big( (\B \X)^T (\B \X) \big)^\dag \Big) (\B \X)^T \B (\B^T \B)^{-1} \nonumber \\
& = \; (\B \X)^T \big((\B \X)^T\big)^{\dag} (\B \X)^T (\B^T)^{-1} \nonumber \\
& = \; (\B \X)^T (\B^T)^{-1}\nonumber \\
& = \; \X^T \textrm{.} \nonumber
\end{align}
\end{proof}

\section{Proof of Theorem~\ref{thm:lower}} \label{sec:proof:lower_bound}

In Section~\ref{sec:proof:lower_bound:lemma} we provide two key lemmas,
and then in Section~\ref{sec:proof:lower_bound:theorem} we prove Theorem~\ref{thm:lower} using the two lemmas.

\subsection{Key Lemmas} \label{sec:proof:lower_bound:lemma}

\begin{lemma} \label{lem:lower_bound_improved}
For an $m\times m$ matrix $\B$ with diagonal entries equal to one and off-diagonal entries equal to $\alpha$,
the error incurred by the modified \nystrom method is lower bounded by
\begin{align}
&\|\B - \tilde{\B}^{\textrm{mod}}_{c} \|_F^2 \nonumber \\
& \qquad \geq\; (1 - \alpha)^2 (m-c)\bigg(1+\frac{2}{c}- (1-\alpha) \frac{1+o(1)}{\alpha c m /2}  \bigg) \textrm{.} \nonumber
\end{align}
\end{lemma}

\begin{proof}
Without loss of generality, we assume the first $c$ column of $\B$ are selected to construct $\C$.
We partition $\B$ and $\C$ as:
\[
\B = \left[
       \begin{array}{cc}
         \W & \B_{2 1}^T \\
         \B_{2 1} & \B_{2 2} \\
       \end{array}
     \right]
\qquad
\textrm{ and }
\qquad
\C = \left[
       \begin{array}{c}
         \W  \\
         \B_{2 1} \\
       \end{array}
     \right] \textrm{.}
\]
Here the matrix $\W$ can be expressed by $\W = (1-\alpha) \I_{c} + \alpha \1_{c} \1_{c}^T$.
We apply the Sherman-Morrison-Woodbury formula
\[
(\A + \B \C \D)^{-1}  = \A^{-1} - \A^{-1} \B (\C^{-1} + \D \A^{-1} \B)^{-1} \D \A^{-1}
\]
to compute $\W^{-1}$, yielding
\begin{equation} \label{eq:lower_bound_improved:pinv_W}
\W^{-1} \; = \; \frac{1}{1-\alpha} \I_c - \frac{\alpha}{(1-\alpha)(1-\alpha+c\alpha)} \1_{c} \1_{c}^T \textrm{.}
\end{equation}
We expand the Moore-Penrose inverse of $\C$ by Lemma~\ref{lem:pinv_partitioned_matrix} and obtain
\[
\C^\dag = \W^{-1} \big( \I_c + \S^T \S \big)^{-1}
\left[
  \begin{array}{cc}
    \I_c & \S^T \\
  \end{array}
\right]
\]
where
\[
\S = \B_{2 1} \W^{-1} = \frac{\alpha}{1-\alpha + c\alpha} \1_{m-c} \1_c^T.
\]
It is easily verified that $\S^T \S = \big( \frac{\alpha}{1-\alpha + c\alpha} \big)^2 (m-c) \1_c \1_c^T$.

Now we express the matrix constructed by the modified \nystrom method in a partitioned form:
\begin{small}
\begin{eqnarray}
\tilde{\B}^{\textrm{mod}}_{c}
&=& \C \C^\dag \B \big(\C^\dag\big)^T \C^T \nonumber \\
&=& \left[
      \begin{array}{c}
        \W \\
        \B_{2 1} \\
      \end{array}
    \right] \W^{-1} \big( \I_c + \S^T \S \big)^{-1}
    \left[
      \begin{array}{c c}
        \I_c & \S^T \\
      \end{array}
    \right]
    \B  \nonumber \\
& & \quad \left[
      \begin{array}{c}
        \I_c \\
        \S \\
      \end{array}
    \right]
     \big( \I_c + \S^T \S \big)^{-1} \W^{-1}
     \left[
      \begin{array}{c}
        \W \\
        \B_{2 1} \\
      \end{array}
    \right]^T \nonumber \\
&=& \left[
      \begin{array}{c}
        \big( \I_c + \S^T \S \big)^{-1} \\
        \B_{2 1}\W^{-1} \big( \I_c + \S^T \S \big)^{-1} \\
      \end{array}
    \right]
    \left[
      \begin{array}{c c}
        \I_c & \S^T \\
      \end{array}
    \right]
    \B\nonumber \\
& & \quad  \left[
      \begin{array}{c}
        \I_c \\
        \S \\
      \end{array}
    \right]
     \left[
      \begin{array}{c}
        \big( \I_c + \S^T \S \big)^{-1} \\
        \B_{2 1}\W^{-1} \big( \I_c + \S^T \S \big)^{-1} \\
      \end{array}
    \right]^T \textrm{.} \label{eq:lower_bound_improved:partitioned_Bimp}
\end{eqnarray}
\end{small}
We then compute the submatrices $\big(\I_c + \S^T \S\big)^{-1}$ and $\B_{2 1} \W^{-1} \big(\I_c + \S^T \S \big)^{-1}$ respectively as follows.
We apply the Sherman-Morrison-Woodbury formula
to compute $\big(\I_c + \S^T \S\big)^{-1}$, yielding
\begin{small}
\begin{eqnarray} \label{eq:lower_bound_improved:IcSS}
\big(\I_c + \S^T \S\big)^{-1}
&=& \bigg( \I_c + \Big( \frac{\alpha}{1-\alpha + c\alpha} \Big)^2 (m-c) \1_c \1_c^T \bigg)^{-1} \nonumber\\
&=& \I_c - \gamma_1 \1_c \1_c^T \textrm{,}
\end{eqnarray}
\end{small}
where
\begin{eqnarray}
\gamma_1 = \frac{m-c}{mc + \big( \frac{1-\alpha}{\alpha}\big)^2 + \frac{2(1-\alpha)c}{\alpha}} . \nonumber
\end{eqnarray}
It follows from (\ref{eq:lower_bound_improved:pinv_W}) and (\ref{eq:lower_bound_improved:IcSS}) that
\begin{small}
\begin{align}
&\W^{-1} \big(\I_c + \S^T \S \big)^{-1}
\;=\; (\gamma_2 \I_c - \gamma_3 \1_c \1_c^T) ( \I_c - \gamma_1 \1_c \1_c^T ) \nonumber\\
& \qquad\qquad\qquad =\; \gamma_2 \I_c + (\gamma_1 \gamma_3 c - \gamma_1 \gamma_2 - \gamma_3) \1_c \1_c^T
\end{align}
\end{small}
where
\[
\gamma_2 = \frac{1}{1-\alpha}  \quad \textrm{ and } \quad
\gamma_3 = \frac{\alpha}{(1-\alpha)(1-\alpha + \alpha c)}.
\]
Then we have that
\begin{align} \label{eq:lower_bound_improved:BWIcSS}
& \!\!\!\!\!\!\!\! \B_{2 1} \W^{-1} \big(\I_c + \S^T \S \big)^{-1}\nonumber\\
& \qquad =\; \alpha \big( \gamma_1 \gamma_3 c^2 - \gamma_3 c - \gamma_1 \gamma_2 c + \gamma_2 \big) \1_{m-c} \1_c^T\nonumber\\
& \qquad \triangleq \; \gamma \1_{m-c} \1_c^T ,
\end{align}
where
\begin{eqnarray}
\gamma &=& \alpha \big( \gamma_1 \gamma_3 c^2 - \gamma_3 c - \gamma_1 \gamma_2 c + \gamma_2 \big) \nonumber\\
&=& \frac{\alpha (\alpha c - \alpha + 1)}{2 \alpha c - 2 \alpha - 2 \alpha^2 c + \alpha^2 + \alpha^2 c m + 1}\textrm{.}
\end{eqnarray}

Since $\B_{2 1}=\alpha \1_{m-c} \1_c^T$ and $\B_{2 2} = (1-\alpha)\I_{m-c}+ \alpha\1_{m-c} \1_{m-c}^T$,
it is easily verified that
\begin{align} \label{eq:lower_bound_improved:middle_term}
&\left[
  \begin{array}{cc}
    \I_c & \S^T \\
  \end{array}
\right]
\B
\left[
  \begin{array}{c}
    \I_c \\
    \S \\
  \end{array}
\right] \nonumber\\
&\qquad\qquad  =
\left[
  \begin{array}{cc}
    \I_c & \S^T \\
  \end{array}
\right]
\left[
       \begin{array}{cc}
         \W & \B_{2 1}^T \\
         \B_{2 1} & \B_{2 2} \\
       \end{array}
\right]
\left[
  \begin{array}{c}
    \I_c \\
    \S \\
  \end{array}
\right] \nonumber\\
&\qquad \qquad = (1-\alpha) \I_c + \lambda \1_c \1_c^T \textrm{,}
\end{align}
where
\begin{small}
\[
\lambda = \frac{\alpha (3 \alpha m - \alpha c - 2 \alpha + \alpha^2 c - 3 \alpha^2 m + \alpha^2 + \alpha^2 m^2 + 1)}{(\alpha c - \alpha + 1)^2}
\]
\end{small}

It follows from (\ref{eq:lower_bound_improved:partitioned_Bimp}), (\ref{eq:lower_bound_improved:IcSS}),
(\ref{eq:lower_bound_improved:BWIcSS}), and (\ref{eq:lower_bound_improved:middle_term}) that
\begin{small}
\begin{align}
&\tilde{\B}^{\textrm{mod}}_{c} \nonumber\\
&=
\left[
  \begin{array}{c}
    \I_c - \gamma_1 \1_c \1_c^T \\
    \gamma \1_{m-c} \1_c^T \\
  \end{array}
\right]
\Big( (1-\alpha) \I_c + \lambda \1_c \1_c^T \Big)
\left[
  \begin{array}{c}
    \I_c - \gamma_1 \1_c \1_c^T \\
    \gamma \1_{m-c} \1_c^T \\
  \end{array}
\right]^T \nonumber\\
&\triangleq
\left[
  \begin{array}{cc}
    \tilde{\B}_{1 1} & \tilde{\B}_{2 1}^T \\
    \tilde{\B}_{2 1} & \tilde{\B}_{2 2} \\
  \end{array}
\right], \nonumber
\end{align}
\end{small}
where
\begin{eqnarray}
\tilde{\B}_{1 1}
& = & (1-\alpha) \I_c + \big[ (1-\gamma_1 c) \nonumber \\
&  &  (\lambda - \lambda \gamma_1 c - (1-\alpha) \gamma_1) - (1-\alpha) \gamma_1 \big] \1_c \1_c^T \nonumber \\
& = & (1-\alpha) \I_c + \eta_1 \1_c \1_c^T , \nonumber \\
\tilde{\B}_{2 1}  & = & \tilde{\A}_{1 2}^T
\; = \; \gamma (1-\gamma_1 c) (1-\alpha + \lambda c) \1_{m-c} \1_c^T \nonumber \\
& = & \eta_2  \1_{m-c} \1_c^T  , \nonumber \\
\tilde{\B}_{2 2}
& = & \gamma^2 c (1-\alpha + \lambda c) \1_{m-c} \1_{m-c}^T \nonumber \\
& = & \eta_3 \1_{m-c} \1_{m-c}^T  , \nonumber
\end{eqnarray}
where
\begin{eqnarray}
\eta_1
& = &
 (1-\gamma_1 c)(\lambda - \lambda \gamma_1 c - (1-\alpha) \gamma_1) - (1-\alpha) \gamma_1  , \nonumber\\
\eta_2
& = & \gamma (1-\gamma_1 c) (1-\alpha + \lambda c),   \nonumber \\
\eta_3
& = & \gamma^2 c (1-\alpha + \lambda c) ,\nonumber
\end{eqnarray}
By dealing with the four blocks of $\tilde{\B}^{\textrm{mod}}_{c}$ respectively, we finally obtain that
\begin{align}
&\|\B - \tilde{\B}^{\textrm{mod}}_{c} \|_F^2  \nonumber\\
& = \|\W - \tilde{\B}_{1 1}\|_F^2 + 2\|\B_{2 1} - \tilde{\B}_{2 1}\|_F^2 + \|\B_{2 2} - \tilde{\B}_{2 2}\|_F^2  \nonumber\\
& =  c^2 (\alpha - \eta_1)^2 + 2 c (m-c) (\alpha - \eta_2)^2  \nonumber\\
& \quad + (m-c)(m-c-1)(\alpha - \eta_3)^2 + (m-c) (1 - \eta_3)^2 \nonumber\\
& = (m - c) (\alpha - 1)^2 \big(\alpha^4 c^2 m^2 - 4 \alpha^4 c^2 m + 4 \alpha^4 c^2  \nonumber\\
& \quad + 2 \alpha^4 c m^2 - 4 \alpha^4 c m + \alpha^4 c + \alpha^4 m - \alpha^4 + 4 \alpha^3 c^2 m \nonumber\\
& \quad - 8 \alpha^3 c^2 + 2 \alpha^3 c m + 2 \alpha^3 c - 2 \alpha^3 m + 2 \alpha^3 + 4 \alpha^2 c^2\nonumber\\
& \quad + 2 \alpha^2 c m  - 7 \alpha^2 c+ \alpha^2 m + 4 \alpha c - 2 \alpha + 1 \big)/ \big(2 \alpha c   \nonumber\\
& \quad - 2 \alpha - 2 \alpha^2 c + \alpha^2 + \alpha^2 c m + 1 \big)^2 \nonumber \\
& = (m-c)(\alpha-1)^2\bigg(1+\frac{2}{c}-\frac{(1-\alpha)}{c}\big(6\alpha c-6\alpha \nonumber \\
& \quad -12\alpha^2 c+6\alpha^3 c+6\alpha^2-2 \alpha^3+3\alpha^2 c^2 -3\alpha^3 c^2 \nonumber \\
& \quad +2\alpha^3 c^2 m + 3\alpha^2 c m- 3\alpha^3 c m + 2 \big)/ \big(2\alpha c - 2\alpha \nonumber\\
& \quad - 2\alpha^2 c +\alpha^2 +\alpha^2 c m+1 \big)^2\bigg)\nonumber\\
& = (m-c)(\alpha-1)^2\bigg(1+\frac{2}{c}- \big(1+o(1)\big) \frac{1-\alpha}{\alpha c m /2}  \bigg) \textrm{.} \nonumber
\end{align}
\end{proof}

\begin{lemma}[{Lemma~19 of \cite{wang2013improving}}] \label{lem:nystrom_residual_new}
Given $m$ and $k$, we let $\B$ be an $\frac{m}{k} \times \frac{m}{k}$ matrix
whose diagonal entries equal to one and off-diagonal entries equal to $\alpha \in [0, 1)$.
We let $\A$ be an $m\times m$ block-diagonal matrix
\begin{equation}\label{eq:construction_bad_nystrom_blk}
\A \; =\;  \diag(\underbrace{\B, \cdots , \B}_{k \textrm{ blocks}}) .
\end{equation}
Let $\A_k$ be the best rank-$k$ approximation to the matrix $\A$,
then we have that
\begin{equation}
\| \A - \A_k \|_F \; =\; (1-\alpha) \sqrt{m-k}  \textrm{.} \nonumber
\end{equation}
\end{lemma}

%
%

\subsection{Proof of the Theorem} \label{sec:proof:lower_bound:theorem}

Now we prove Theorem~\ref{thm:lower} using Lemma~\ref{lem:lower_bound_improved} and Lemma~\ref{lem:nystrom_residual_new}.

\begin{proof}
Let $\C$ consist of $c$ column sampled from $\A$ and $\hat{\C}_i$ consist of $c_i$ columns sampled from the $i$-th block diagonal matrix in $\A$.
Without loss of generality, we assume $\hat{\C}_i$ consists of the first $c_i$ columns of $\B$.
Then the intersection matrix $\U$ is computed by
\begin{eqnarray}
\U
& = & \C^\dag \A \big(\C^T\big)^\dag \nonumber\\
& = & \big[ \diag\big( \hat{\C}_1 , \cdots , \hat{\C}_k \big) \big]^\dag
        \A
        \big[ \diag\big( \hat{\C}_1^T , \cdots , \hat{\C}_k^T \big) \big]^\dag \nonumber\\
& = & \diag\Big( \hat{\C}_1^\dag \B \big(\hat{\C}_1^\dag\big)^T , \cdots , \hat{\C}_k^\dag \B \big(\hat{\C}_k^\dag\big)^T \Big)  \textrm{.} \nonumber
\end{eqnarray}
The modified \nystrom approximation to $\A$ is
\begin{align}
&\tilde{\A}^{\textrm{mod}}_{c}
\; = \; \C \U \C^T \nonumber\\
& = \; \diag\Big( \hat{\C}_1 \hat{\C}_1^\dag \B \big(\hat{\C}_1^\dag\big)^T \hat{\C}_1^T,
        \cdots , \hat{\C}_k \hat{\C}_k^\dag \B \big(\hat{\C}_k^\dag\big)^T \hat{\C}_k^T \Big) \textrm{,} \nonumber
\end{align}
and thus the approximation error is
\begin{footnotesize}
\begin{align}
& \big\|\A - \tilde{\A}^{\textrm{mod}}_{c} \big\|_F^2
\; = \; \sum_{i=1}^{k} \Big\| \B \;-\; \hat{\C}_i \hat{\C}_i^\dag \B \big(\hat{\C}_i^\dag\big)^T \hat{\C}_i^T \Big\|_F^2 \nonumber\\
& \geq \; (1 - \alpha)^2 \sum_{i=1}^{k}(p - c_i)\bigg(1+\frac{2}{c_i}- (1-\alpha) \Big(\frac{1+o(1)}{\alpha c_i p /2} \Big) \bigg) \nonumber \\
& = \; (1-\alpha)^2 \bigg( \sum_{i=1}^{k}(p - c_i) \nonumber \\
& \qquad \qquad + \sum_{i=1}^{k}\frac{2(p - c_i)}{c_i}\Big( 1 - \frac{(1-\alpha)(1+o(1))}{\alpha  p} \Big) \bigg) \nonumber \\
& \geq \;(1-\alpha)^2 (m - c) \bigg( 1 + \frac{2k}{c}\Big( 1 - \frac{k(1-\alpha)(1+o(1))}{\alpha m} \Big) \bigg)\textrm{,} \nonumber
\end{align}
\end{footnotesize}
where the former inequality follows from Lemma~\ref{lem:lower_bound_improved},
and the latter inequality follows by minimizing over $c_1 , \cdots , c_k$.
Finally we apply Lemma~\ref{lem:nystrom_residual_new},
and the theorem follows by setting $\alpha \rightarrow 1$.
\end{proof}

\section{Supplementary Experiments}

We have mentioned in Remark~\ref{remark:efficient} that the resulting approximation accuracy is insensitive to
the parameter $\mu$ in Algorithm~\ref{alg:efficient},
and setting $\mu$ to be exactly the matrix coherence does not in general give rise to the highest accuracy.
To demonstrate this point of view,
we conduct experiments on an RBF kernel matrix of the Letters Dataset with $\sigma=0.2$, and we set $k=10$.

We compare the uniform+adaptive$^2$ algorithm with different settings of $\mu$;
we also employ the adaptive-full algorithm of \cite{kumar2012sampling},
the near-optimal+adaptive algorithm of \cite{wang2013improving},
and the uniform sampling algorithm for comparison.
The experiment settings are the same to Section~\ref{sec:experiments}.
Here the adaptive-full algorithm also has three steps: one uniform sampling and two adaptive sampling steps,
and we set $c_1 = c_2  = c_3 = c/3$ according to \cite{kumar2012sampling}.
We plot the approximation errors in Figure~\ref{fig:mu}.

\begin{figure}[ht]
\begin{center}
\centering
\includegraphics[width=80mm]{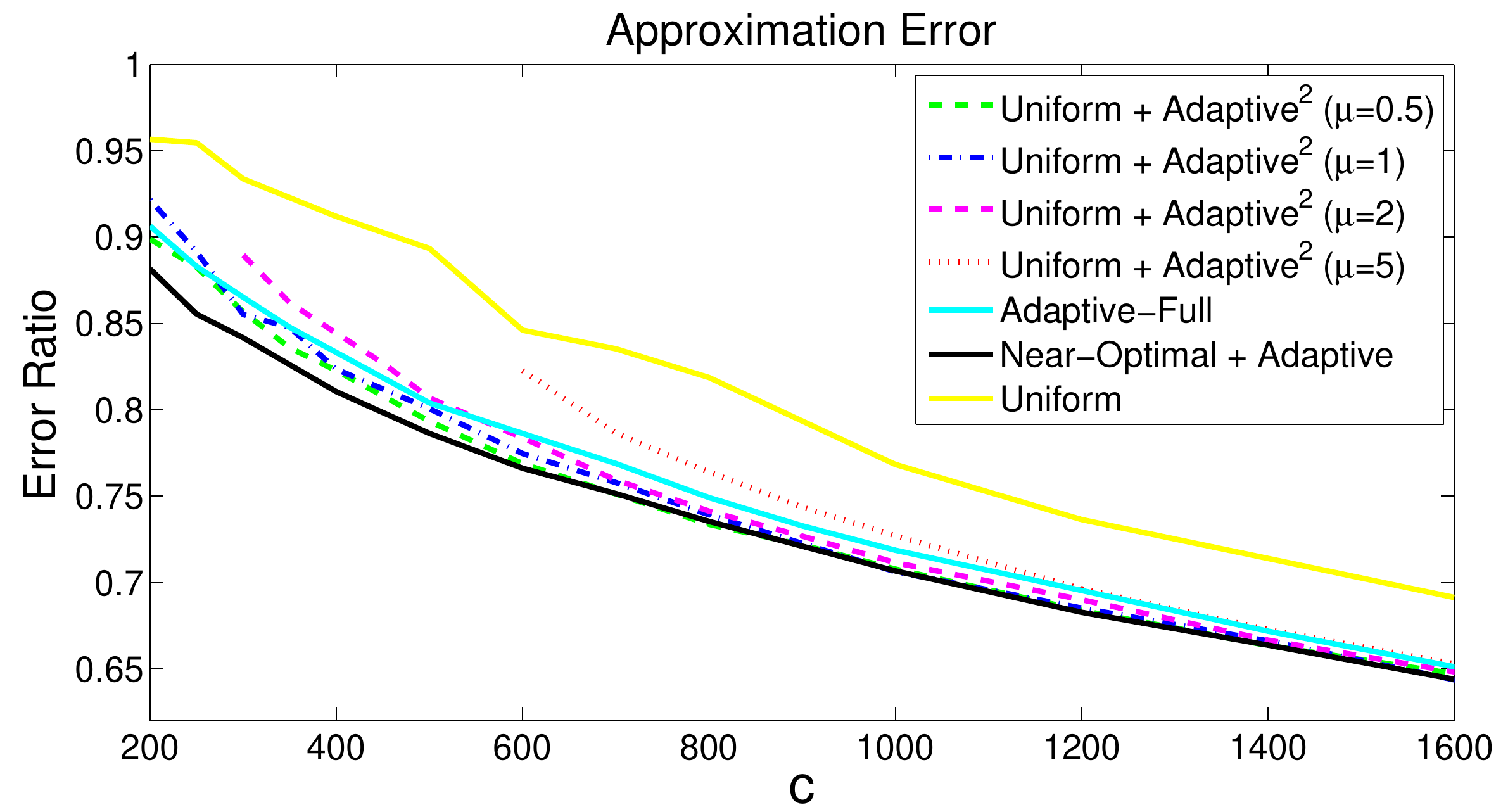}
\end{center}
   \caption{Effect of the parameter $\mu$ in Algorithm~\ref{alg:efficient}.}
\label{fig:mu}
\end{figure}

We can see from Figure~\ref{fig:mu} that different settings of $\mu$ does not have big influence on the approximation accuracy.
We can also see that it is unnecessary to set $\mu$ to be exactly the matrix coherence;
in this set of experiments, the uniform+adaptive$^2$ algorithm achieves the higher accuracy
when $\mu=0.5$ (the actual matrix coherence is $\mu_{10} = 62.05$).

\end{document}